\documentclass[sigconf]{acmart}

\usepackage{booktabs} 

\setcopyright{rightsretained}

\usepackage{amsmath}
\usepackage{mathrsfs}
\usepackage{diagbox}
\usepackage[lined,boxed,commentsnumbered,ruled]{algorithm2e}
\usepackage{algorithmic}
\usepackage{color}
\usepackage{subfigure}
\usepackage{makecell}
\usepackage{bm}
\usepackage{url}
\usepackage{multirow}

\newcommand{\rev}[1]{{#1}} 
\newcommand{\nop}[1]{}

\begin{document}
\begin{sloppy}

\copyrightyear{2018} 
\acmYear{2018} 
\setcopyright{acmcopyright}
\acmConference[KDD '18]{The 24th ACM SIGKDD International Conference on Knowledge Discovery \& Data Mining}{August 19--23, 2018}{London, United Kingdom}
\acmBooktitle{KDD '18: The 24th ACM SIGKDD International Conference on Knowledge Discovery \& Data Mining, August 19--23, 2018, London, United Kingdom}
\acmPrice{15.00}
\acmDOI{10.1145/3219819.3220063}
\acmISBN{978-1-4503-5552-0/18/08}

\fancyhead{}

\title{Exact and Consistent Interpretation for Piecewise Linear \\ Neural Networks: A Closed Form Solution}
\titlenote{This work was supported in part by the NSERC Discovery Grant program, the Canada Research Chair program, the NSERC Strategic Grant program. All opinions, findings, conclusions and recommendations in this paper are those of the authors and do not necessarily reflect the views of the funding agencies.}

\author{Lingyang Chu}
\affiliation{%
  \institution{Simon Fraser University}
  \city{Burnaby}
  \country{Canada}
}
\email{lca117@sfu.ca}

\author{Xia Hu}
\affiliation{%
  \institution{Simon Fraser University}
  \city{Burnaby}
  \country{Canada}
}
\email{huxiah@sfu.ca}

\author{Juhua Hu}
\affiliation{%
  \institution{Simon Fraser University}
  \city{Burnaby}
  \country{Canada}
}
\email{juhuah@sfu.ca}

\author{Lanjun Wang}
\affiliation{%
  \institution{Huawei Technology Co. Ltd}
  \city{Beijing}
  \country{China}
}
\email{lanjun.wang@huawei.com}

\author{Jian Pei	}
\affiliation{
  \institution{JD.com and Simon Fraser University}
  \city{Beijing/Burnaby} 
  \country{China/Canada}}
\email{jpei@cs.sfu.ca}

\nop{
Strong intelligent machines powered by deep neural networks have been widely used as a black box to make decisions in many real world applications.

have been widely used as black boxes to make important decisions in many real world applications.
The spiking impact of decision making machines gives rise to an urgent demand in clearly interpreting each machine-made decision.

}

\begin{abstract}
Strong intelligent machines powered by deep neural networks are increasingly deployed as black boxes to make decisions in risk-sensitive domains, such as finance and medical.
To reduce potential risk and build trust with users, it is critical to interpret how such machines make their decisions.
Existing works interpret a pre-trained neural network by analyzing hidden neurons, mimicking pre-trained models or approximating local predictions.
However, these methods do not provide a guarantee on the exactness and consistency of their interpretations.
In this paper, we propose an elegant closed form solution named $OpenBox$ to compute exact and consistent interpretations for the family of Piecewise Linear Neural Networks (PLNN).
The major idea is to first transform a PLNN into a mathematically equivalent set of linear classifiers, then interpret each linear classifier by the features that dominate its prediction.
We further apply $OpenBox$ to demonstrate the effectiveness of non-negative and sparse constraints on improving the interpretability of PLNNs.
The extensive experiments on both synthetic and real world data sets clearly demonstrate the exactness and consistency of our interpretation.
\end{abstract}

\keywords{Deep neural network; exact and consistent interpretation; closed form.}

\maketitle

\section{Introduction}
\label{sec:intro}

More and more machine learning systems are making significant decisions routinely in important domains, such as medical practice, autonomous driving, criminal justice, and military decision making~\cite{Goodfellow-et-al-2016}. As the impact of machine-made decisions increases,  the demand on clear interpretations of machine learning systems is growing ever stronger against the blind deployments of decision machines~\cite{goodman2016european}.
Accurately and reliably interpreting a machine learning model is the key to many significant tasks, such as identifying failure models~\cite{agrawal2016analyzing}, building trust with human users~\cite{ribeiro2016should}, discovering new knowledge~\cite{rather2017using}, and avoiding unfairness issues~\cite{zemel2013learning}.

The interpretation problem of machine learning models has been studied for decades.
Conventional models, such as Logistic Regression and Support Vector Machine, have all been well interpreted from both practical and theoretical perspectives~\cite{bishop2007pattern}.
Powerful non-negative and sparse constraints are also developed to enhance the interpretability of conventional models by sparse feature selection~\cite{lee2007efficient, hoyer2002non}.
However, due to the complex network structure of a deep neural network, the interpretation problem of modern deep models is yet a challenging field that awaits further exploration.

As to be reviewed in Section~\ref{sec:rw}, the existing studies interpret a deep neural network in three major ways.
The hidden neuron analysis methods~\cite{mahendran2015understanding, yosinski2015understanding, dosovitskiy2016inverting} analyze and visualize the features learned by the hidden neurons of a neural network;
the model mimicking methods~\cite{ba2014deep, che2015distilling, hinton2015distilling, bastani2017interpreting} build a transparent model to imitate the classification function of a deep neural network;
the local explanation methods~\cite{shrikumar2017learning, fong2017interpretable, sundararajan2017axiomatic, smilkov2017smoothgrad} study the predictions on local perturbations of an input instance, so as to provide decision features for interpretation.
All these methods gain useful insights into the mechanism of deep models. 
However, there is no guarantee that what they compute as an interpretation is truthfully the exact behavior of a deep neural network.
As demonstrated by Ghorbani~\cite{ghorbani2017interpretation}, most existing interpretation methods are inconsistent and fragile, because two perceptively indistinguishable instances with the same prediction result can be easily manipulated to have dramatically different interpretations.

\emph{
Can we compute an exact and consistent interpretation for a pre-trained deep neural network? 
} 
\nop{Explain the intuition of ``exact'' and ``consistent''.}
In this paper, we provide an affirmative answer, as well as an elegant closed form solution for the family of piecewise linear neural networks.
Here, a \textbf{piecewise linear neural network} (\textbf{PLNN})~\cite{harvey2017nearly} is a neural network that adopts a piecewise linear activation function, such as MaxOut~\cite{goodfellow2013maxout} and the family of ReLU~\cite{glorot2011deep, nair2010rectified,he2015delving}. The wide applications~\cite{lecun2015deep} and great practical successes~\cite{krizhevsky2012imagenet} of PLNNs call for exact and consistent interpretations on the overall behaviour of this type of neural networks.
We make the following technical contributions.

First, we prove that a PLNN is mathematically equivalent to a set of local linear classifiers, each of which being a linear classifier that classifies a group of instances within a convex polytope in the input space.
Second, we propose a method named $OpenBox$ to provide an exact interpretation of a PLNN by computing its equivalent set of local linear classifiers in closed form.
Third, we interpret the classification result of each instance by the decision features of its local linear classifier. Since all instances in the same convex polytope share the same local linear classifier, our interpretations are consistent per convex polytope.
Fourth, we also apply $OpenBox$ to study the effect of non-negative and sparse constraints on the interpretability of PLNNs. We find that a PLNN trained with these constraints selects meaningful features that dramatically improve the interpretability.
Last, we conduct extensive experiments on both synthetic and real-world data sets to verify the effectiveness of our method.

The rest of this paper is organized as follows. 
We review the related works in Section~\ref{sec:rw}.
We formulate the problem in Section~\ref{sec:prob} and present $OpenBox$ in Section~\ref{sec:obm}.
We report the experimental results in Section~\ref{sec:exp}, and conclude the paper in Section~\ref{sec:con}.

\section{Related Works}
\label{sec:rw}
How to interpret the overall mechanism of deep neural networks is an emergent and challenging problem.

\subsection{Hidden Neuron Analysis Methods}
The hidden neuron analysis methods~\cite{mahendran2015understanding, yosinski2015understanding, dosovitskiy2016inverting} interpret a pre-trained deep neural network by visualizing, revert-mapping or labeling the features that are learned by the hidden neurons.

Yosinski~\textit{et al.}~\cite{yosinski2015understanding} visualized the live activations of the hidden neurons of a ConvNet, and proposed a regularized optimization to produce a qualitatively better visualization\nop{ for the features at each layer of a DNN}.
Erhan~\textit{et al.}~\cite{erhan2009visualizing} proposed an activation maximization method and a unit sampling method to visualize the features learned by hidden neurons.
\nop{Based on the goal of a neural network, }Cao~\textit{et~al.}~\cite{cao2015look} visualized a neural network's attention on its target objects by a feedback loop that infers the activation status of the hidden neurons.
Li~\textit{et al.}~\cite{li2015visualizing} visualized the compositionality of clauses by analyzing the outputs of hidden neurons in a neural model for Natural Language Processing.

To understand the features learned by the hidden neurons, Mahendran~\textit{et al.}~\cite{mahendran2015understanding} proposed a general framework that revert-maps the features learned from an image to reconstruct the image.
Dosovitskiy~\textit{et al.}~\cite{dosovitskiy2016inverting} performed the same task as Mahendran~\textit{et al.}~\cite{mahendran2015understanding} did by training an up-convolutional neural network.

Zhou~\textit{et al.}~\cite{zhou2017interpreting} interpreted a CNN by labeling each hidden neuron with a best aligned human-understandable semantic concept.
However, it is hard to get a golden dataset with accurate and complete labels of all human semantic concepts.

The hidden neuron analysis methods provide useful qualitative insights into the properties of each hidden neuron.
However, qualitatively analyzing every neuron does not provide much actionable and quantitative interpretation about the overall mechanism of the entire neural network~\cite{frosst2017distilling}.

\subsection{Model Mimicking Methods}
By imitating the classification function of a neural network, the model mimicking methods~\cite{ba2014deep, che2015distilling, hinton2015distilling, bastani2017interpreting} build a transparent model that is easy to interpret and achieves a high classification accuracy.

Ba~\textit{et al.}~\cite{ba2014deep} proposed a model compression method to train a shallow mimic network using the training instances labeled by one or more deep neural networks.
Hinton~\textit{et al.}~\cite{hinton2015distilling} proposed a distillation method that distills the knowledge of a large neural network by training a relatively smaller network to mimic the prediction probabilities of the original large network. 
To improve the interpretability of distilled knowledge, Frosst and Hinton~\cite{frosst2017distilling} extended the distillation method~\cite{hinton2015distilling} by training a soft decision tree to mimic the prediction probabilities of a deep neural network.

Che~\textit{et al.}~\cite{che2015distilling} proposed a mimic learning method to learn interpretable phenotype features.
Wu~\textit{et al.}~\cite{wu2017beyond} proposed a tree regularization method that uses a binary decision tree to mimic and regularize the classification function of a deep time-series model.
\rev{
Zhu~\textit{et al.}~\cite{zhu2017deep} built a transparent forest model on top of a deep feature embedding network, however it is still difficult to interpret the deep feature embedding network.
}

The mimic models built by model mimicking methods are much simpler to interpret than deep neural networks. 
However, due to the reduced model complexity of a mimic model, there is no guarantee that a deep neural network with a large VC-dimension~\cite{sontag1998vc, koiran1996neural, harvey2017nearly} can be successfully imitated by a simpler shallow model.
Thus, there is always a gap between the interpretation of a mimic model and the actual overall mechanism of the target deep neural network.

\subsection{Local Interpretation Methods}
The local interpretation methods~\cite{shrikumar2017learning, fong2017interpretable, sundararajan2017axiomatic, smilkov2017smoothgrad} compute and visualize the important features for an input instance by analyzing the predictions of its local perturbations.

Simonyan~\textit{et al.}~\cite{simonyan2013deep} generated a class-representative image and a class-saliency map for each class of images by computing the gradient of the class score with respect to an input image.
Ribeiro~\textit{et al.}~\cite{ribeiro2016should} proposed LIME to interpret the predictions of any classifier by learning an interpretable model in the local region around the input instance.

Zhou~\textit{et al.}~\cite{zhou2016learning} proposed CAM to identify discriminative image regions for each class of images using the global average pooling in CNNs.
Selvaraju~\textit{et al.}~\cite{selvaraju2016grad} generalized CAM~\cite{zhou2016learning} by Grad-CAM, which identifies important regions of an image by flowing class-specific gradients into the final convolutional layer of a CNN.

Koh~\textit{et al.}~\cite{koh2017understanding} used influence functions to trace a model's prediction and identify the training instances that are the most responsible for the prediction.

The local interpretation methods generate an insightful individual interpretation for each input instance. 
However, the interpretations for perspectively indistinguishable instances may not be consistent~\cite{ghorbani2017interpretation}, and can be manipulated by a simple transformation of the input instance without affecting the prediction result~\cite{kindermans2017reliability}.


\section{Problem Definition}
\label{sec:prob}


\nop{Examples of PLNN? Why is PLNN important and why do we use PLNN as the subject in this paper?}

For a PLNN $\mathcal{N}$ that contains $L$ layers of neurons, we write the $l$-th layer of $\mathcal{N}$ as $\mathcal{L}_l$. 
Hence, $\mathcal{L}_1$ is the \textbf{input layer}, $\mathcal{L}_L$ is the \textbf{output layer}, and the other layers $\mathcal{L}_l$, $l\in\{2, \ldots, L-1\}$ are \textbf{hidden layers}.
A neuron in a hidden layer is called a \textbf{hidden neuron}.
Let $n_l$ represent the number of neurons in $\mathcal{L}_l$, the total number of hidden neurons in $\mathcal{N}$ is computed by $N=\sum_{l=2}^{L-1} n_l$.

Denote by $\mathbf{u}_i^{(l)}$ the $i$-th neuron in $\mathcal{L}_l$, by $\mathbf{b}^{(l-1)}_i$ its bias, by $\mathbf{a}^{(l)}_i$ its output, and by $\mathbf{z}^{(l)}_i$ the total weighted sum of its inputs. 
For all the $n_l$ neurons in $\mathcal{L}_l$, we write their biases as a vector $\mathbf{b}^{(l-1)} = [\mathbf{b}^{(l-1)}_1, \ldots, \mathbf{b}^{(l-1)}_{n_l}]^\top$, their outputs as a vector $\mathbf{a}^{(l)} = [\mathbf{a}^{(l)}_1, \ldots, \mathbf{a}^{(l)}_{n_l}]^\top$, and their inputs as a vector $\mathbf{z}^{(l)} = [\mathbf{z}^{(l)}_1, \ldots, \mathbf{z}^{(l)}_{n_l}]^\top$.

Neurons in successive layers are connected by weighted edges.
Denote by $W^{(l)}_{ij}$ the weight of the edge between the $i$-th neuron in $\mathcal{L}_{l+1}$ and the $j$-th neuron in $\mathcal{L}_{l}$, that is, $W^{(l)}$ is an $n_{l+1}$-by-$n_l$ matrix.
For $l\in\{1, \ldots, L-1\}$, we compute $\mathbf{z}^{(l+1)}$ by
\begin{equation}
\label{eqn:nextinput}
	\mathbf{z}^{(l+1)} = W^{(l)}\mathbf{a}^{(l)} + \mathbf{b}^{(l)}
\end{equation}

Denote by $f: \mathbb{R} \rightarrow \mathbb{R}$ the piecewise linear activation function for each neuron in the hidden layers of $\mathcal{N}$. We have $\mathbf{a}^{(l)}_i = f(\mathbf{z}^{(l)}_i)$ for all $l\in\{2, \ldots, L-1\}$.
We extend $f$ to apply to vectors in an element-wise fashion, such that $f(\mathbf{z}^{(l)}) = [f(\mathbf{z}^{(l)}_1), \ldots, f(\mathbf{z}^{(l)}_{n_l})]^\top$. 
Then, we compute $\mathbf{a}^{(l)}$ for all $l\in\{2, \ldots, L-1\}$ by 
\begin{equation}
\label{eqn:nextoutput}
	\mathbf{a}^{(l)} = f(\mathbf{z}^{(l)})
\end{equation}

An \textbf{input instance} of $\mathcal{N}$ is denoted by $\mathbf{x}\in \mathcal{X}$, where $\mathcal{X} \subseteq \mathbb{R}^d$ is a $d$-dimensional input space. $\mathbf{x}$ is also called an \textbf{instance} for short. 

Denote by $\mathbf{x}_i$ the $i$-th dimension of $\mathbf{x}$.
The input layer $\mathcal{L}_1$ contains $n_1 = d$ neurons, where $\mathbf{a}^{(1)}_i = \mathbf{x}_i$ for all $i\in\{1, \ldots, d\}$. 

\nop{That is, $\mathbf{a}^{(1)} = \mathbf{x}$.}

The \textbf{output} of $\mathcal{N}$ is $\mathbf{a}^{(L)}\in \mathcal{Y}$, where $\mathcal{Y}\subseteq \mathbb{R}^{n_L}$ is an $n_L$-dimensional output space.
The output layer $\mathcal{L}_L$ adopts the \emph{softmax} function to compute the output by $\mathbf{a}^{(L)} = \text{\emph{softmax}}(\mathbf{z}^{(L)})$.

A PLNN works as a \textbf{classification function} $F: \mathcal{X} \rightarrow \mathcal{Y}$ that maps an input $\mathbf{x}\in \mathcal{X}$ to an output $\mathbf{a}^{(L)}\in \mathcal{Y}$.
It is widely known that $F(\cdot)$ is a piecewise linear function~\cite{pascanu2013number, montufar2014number}.
However, due to the complex network of a PLNN, the overall behaviour of $F(\cdot)$ is hard to understand. Thus, a PLNN is usually regarded as a black box.

How to interpret the overall behavior of a PLNN in a human-understandable manner is an interesting problem that has attracted much attention in recent years.

\nop{Define what an interpretation is.}

Following a principled approach of interpreting a machine learning model~\cite{bishop2007pattern}, we regard an \textbf{interpretation} of a PLNN $\mathcal{N}$ as the decision features that define the decision boundary of $\mathcal{N}$.
We call a model \textbf{interpretable} if it explicitly provides its interpretation (i.e., decision features) in closed form.

\begin{definition}
\label{prob:interpret}
Given a fixed PLNN $\mathcal{N}$ with constant structure and parameters, our task is to interpret the overall behaviour of $\mathcal{N}$ by computing an interpretable model $\mathcal{M}$ that satisfies the following requirements.
\begin{itemize}
\item \textbf{Exactness:} $\mathcal{M}$ is mathematically equivalent to $\mathcal{N}$ such that the interpretations provided by $\mathcal{M}$ truthfully describe the exact behaviour of $\mathcal{N}$.
\item \textbf{Consistency:} $\mathcal{M}$ provides similar interpretations for classification of similar instances.
\end{itemize}
\end{definition}

Table~\ref{Table:notations} summarizes a list of frequently used notations.

\newcommand{\NotationTWidth}{63mm}
\begin{table}[t]
\centering\small
\caption{Frequently used notations.}
\label{Table:notations}
\begin{tabular}{|c|p{\NotationTWidth}|}
\hline
\makecell[c]{Notation}       &      \makecell[c]{Description} \\ \hline

\nop{
$\mathcal{N}$
&   The pre-trained fixed PLNN. \\ \hline

$\mathcal{L}_l$
&   The $l$-th layer of the PLNN $\mathcal{N}$. \\ \hline
}

$\mathbf{u}_i^{(l)}$
&   The $i$-th neuron in layer $\mathcal{L}_l$. \\ \hline

$n_l$
&   The number of neurons in layer $\mathcal{L}_l$. \\ \hline

$N$
&   The total number of hidden neurons in $\mathcal{N}$. \\ \hline

\nop{
$\mathbf{b}_i^{(l-1)}$
&   The bias associated with the $i$-th neuron in layer $\mathcal{L}_l$. \\ \hline

$\mathbf{a}_i^{(l)}$
&   The output of the $i$-th neuron in layer $\mathcal{L}_l$. \\ \hline
}

$\mathbf{z}_i^{(l)}$
&   The input of the $i$-th neuron in layer $\mathcal{L}_l$. \\ \hline

$\mathbf{c}^{(l)}_i$
&   The configuration of the $i$-th neuron in layer $\mathcal{L}_l$. \\ \hline

$\mathbf{C}_h$
&   The $h$-th configuration of the PLNN $\mathcal{N}$. \\ \hline

$P_h$
&   The $h$-th convex polytope determined by $\mathbf{C}_h$. \\ \hline

$F_h(\cdot)$
&   The $h$-th linear classifier that is determined by $\mathbf{C}_h$. \\ \hline

$Q_h$
&   The set of linear inequalities that define $P_h$. \\ \hline

\nop{
$\mathbb{R}^d$
&   The $d$-dimensional Euclidean space. \\ \hline
}

\end{tabular}
\end{table}

\section{The OpenBox Method}
\label{sec:obm}
In this section, we describe the  $OpenBox$ method, which produces an exact and consistent interpretation of a PLNN by computing an interpretation model $\mathcal{M}$ in a piecewise linear closed form.

\nop{its $F(\cdot)$ in an understandable closed form.}

We first define the configuration of a PLNN $\mathcal{N}$, which specifies the activation status of each hidden neuron in $\mathcal{N}$. 
Then, we illustrate how to interpret the classification result of a fixed instance. 
Last, we illustrate how to interpret the overall behavior of $\mathcal{N}$ by computing an interpretation model $\mathcal{M}$ that is mathematically equivalent to $\mathcal{N}$.

\nop{At last, we introduce a non-negative and sparse PLNN that achieves a better interpretation by adding non-negative and sparse constraints on the edge weights of $\mathcal{N}$.}

\subsection{The Configuration of a PLNN}
For a hidden neuron $\mathbf{u}_i^{(l)}$, the piecewise linear activation function $f(\mathbf{z}_i^{(l)})$ is in the following form.
\begin{equation}
\label{eqn:plf}
f(\mathbf{z}_i^{(l)}) = 
    \left \{
        \begin{aligned}
            & r_1 \mathbf{z}_i^{(l)} + t_1, \quad \text{if } \mathbf{z}_i^{(l)} \in I_1 \cr
            & r_2 \mathbf{z}_i^{(l)} + t_2, \quad \text{if } \mathbf{z}_i^{(l)} \in I_2 \cr
            & \quad \quad \quad \vdots \cr
            & r_k \mathbf{z}_i^{(l)} + t_k, \quad \text{if } \mathbf{z}_i^{(l)} \in I_k \cr
        \end{aligned}
    \right.
\end{equation}
where $k\geq 1$ is a constant integer, $f(\mathbf{z}_i^{(l)})$ consists of $k$ linear functions, $\{r_1, \ldots, r_k\}$ are constant \textbf{slopes}, $\{t_1, \ldots, t_k\}$ are constant \textbf{intercepts}, and $\{I_1, \ldots, I_k\}$ is a collection of constant \textbf{real intervals} that partition $\mathbb{R}$.

Given a fixed PLNN $\mathcal{N}$, an instance $\mathbf{x}\in\mathcal{X}$ determines the value of $\mathbf{z}_i^{(l)}$, and further determines a linear function in $f(\mathbf{z}_i^{(l)})$ to apply.
According to which linear function in $f(\mathbf{z}_i^{(l)})$ is applied, we encode the activation status of each hidden neuron by $k$ \textbf{states}, each of which uniquely corresponds to one of the $k$ linear functions of $f(\mathbf{z}_i^{(l)})$.
Denote by $\mathbf{c}^{(l)}_i \in \{1, \ldots, k\}$ the state of $\mathbf{u}_i^{(l)}$, we have $\mathbf{z}_i^{(l)}\in I_q$ if and only if $\mathbf{c}_i^{(l)} = q$ ($q\in\{1,\ldots,k\}$). 
Since the inputs $\mathbf{z}_i^{(l)}$'s are different from neuron to neuron, the states of different hidden neurons may differ from each other.

Denote by a vector $\mathbf{c}^{(l)}=[\mathbf{c}^{(l)}_1, \ldots, \mathbf{c}^{(l)}_{n_l}]$ the states of all hidden neurons in $\mathcal{L}_l$.
The \textbf{configuration} of $\mathcal{N}$ is an $N$-dimensional vector, denoted by $\mathbf{C} = [\mathbf{c}^{(2)}, \ldots, \mathbf{c}^{(L-1)}]$, which specifies the states of all hidden neurons in $\mathcal{N}$.

The configuration $\mathbf{C}$ of a fixed PLNN is uniquely determined by the instance $\mathbf{x}$. We write the function that maps an instance $\mathbf{x}\in\mathcal{X}$ to a configuration $\mathbf{C}\in\{1, \ldots, k\}^N$ as $\text{\emph{conf}}: \mathcal{X} \rightarrow \{1,\ldots,k\}^N$.

For a neuron $\mathbf{u}_i^{(l)}$, denote by variables $\mathbf{r}^{(l)}_i$ and $\mathbf{t}^{(l)}_i$ the slope and intercept, respectively, of the linear function that corresponds to the state $\mathbf{c}_i^{(l)}$.
$\mathbf{r}^{(l)}_i$ and $\mathbf{t}^{(l)}_i$ are uniquely determined by $\mathbf{c}_i^{(l)}$, such that $\mathbf{r}_i^{(l)} = r_q$ and $\mathbf{t}_i^{(l)} = t_q$, if and only if $\mathbf{c}_i^{(l)}=q$ ($q\in \{1,\ldots,k\}$).

For all hidden neurons in $\mathcal{L}_l$, we write the variables of slopes and intercepts as $\mathbf{r}^{(l)} = [\mathbf{r}_1^{(l)}, \ldots, \mathbf{r}_{n_l}^{(l)}]^\top$ and $\mathbf{t}^{(l)} = [\mathbf{t}_1^{(l)}, \ldots, \mathbf{t}_{n_l}^{(l)}]^\top$, respectively.
Then, we rewrite the activation function for all neurons in a hidden layer $\mathcal{L}_l$ as 
\begin{equation}
\label{eqn:plfsimple}
	f(\mathbf{z}^{(l)}) = \mathbf{r}^{(l)} \circ \mathbf{z}^{(l)} + \mathbf{t}^{(l)}
\end{equation}
where $\mathbf{r}^{(l)} \circ \mathbf{z}^{(l)}$ is the Hadamard product between $\mathbf{r}^{(l)}$ and $\mathbf{z}^{(l)}$.

Next, we interpret the classification result of a fixed instance.


\subsection{Exact Interpretation for the Classification Result of a Fixed Instance}
\label{sec:eecrfi}
Given a fixed PLNN $\mathcal{N}$, we interpret the classification result of a fixed instance $\mathbf{x}\in\mathcal{X}$ by deriving the closed form of $F(\mathbf{x})$ as follows.

Following Equations~\ref{eqn:nextoutput} and~\ref{eqn:plfsimple}, we have, for all $l\in\{2, \ldots, L-1\}$ 
\nop{[For all $l\in\{2, \ldots, L-1\}$, it follows Equation~\ref{eqn:nextoutput} and Equation~\ref{eqn:plfsimple} that]}
\begin{equation}
\nonumber
	\mathbf{a}^{(l)} = f(\mathbf{z}^{(l)}) = \mathbf{r}^{(l)} \circ \mathbf{z}^{(l)} + \mathbf{t}^{(l)}
\end{equation}

By plugging $\mathbf{a}^{(l)}$ into Equation~\ref{eqn:nextinput}, we rewrite $\mathbf{z}^{(l+1)}$ as
\begin{equation}
\label{eqn:genterm}
\begin{aligned}
	\mathbf{z}^{(l+1)} 
	\nop{& = W^{(l)} \mathbf{a}^{(l)} + \mathbf{b}^{(l)} \\}
	& = W^{(l)} (\mathbf{r}^{(l)} \circ \mathbf{z}^{(l)} + \mathbf{t}^{(l)}) + \mathbf{b}^{(l)} 
	\nop{& = (W^{(l)} \circ \mathbf{r}^{(l)}) \mathbf{z}^{(l)} + W^{(l)} \mathbf{t}^{(l)} + \mathbf{b}^{(l)} \\}
	 = \tilde{W}^{(l)} \mathbf{z}^{(l)} + \tilde{\mathbf{b}}^{(l)} \\
\end{aligned}
\end{equation}
where $\tilde{\mathbf{b}}^{(l)} = W^{(l)} \mathbf{t}^{(l)} + \mathbf{b}^{(l)}$, and $\tilde{W}^{(l)} = W^{(l)} \circ \mathbf{r}^{(l)}$ is an extended version of Hadamard product, such that the entry at the $i$-th row and $j$-th column of $\tilde{W}^{(l)}$ is $\tilde{W}^{(l)}_{ij} = W_{ij}^{(l)} \mathbf{r}_{j}^{(l)}$.

By iteratively plugging Equation~\ref{eqn:genterm} into itself, we can write $\mathbf{z}^{(l+1)}$ for all $l \in \{2, \ldots, L-1\}$ as
\begin{equation}
\nonumber
	\mathbf{z}^{(l+1)} = \prod_{h=0}^{l-2} \tilde{W}^{(l-h)} \mathbf{z}^{(2)} + \sum_{h=2}^{l}\prod_{q=0}^{l-h-1} \tilde{W}^{(l-q)}\tilde{\mathbf{b}}^{(h)}
\end{equation}

By plugging $\mathbf{z}^{(2)} = W^{(1)} \mathbf{a}^{(1)} + \mathbf{b}^{(1)}$ and $\mathbf{a}^{(1)} = \mathbf{x}$ into the above equation, we rewrite $\mathbf{z}^{(l+1)}$, for all $l \in \{2, \ldots, L-1\}$, as
\begin{equation}
\label{eqn:simpleform}
\begin{aligned}
	\mathbf{z}^{(l+1)} 
	& = \prod_{h=0}^{l-2} \tilde{W}^{(l-h)} W^{(1)} \mathbf{x} + \prod_{h=0}^{l-2} \tilde{W}^{(l-h)} \mathbf{b}^{(1)} + \sum_{h=2}^{l}\prod_{q=0}^{l-h-1} \tilde{W}^{(l-q)}\tilde{\mathbf{b}}^{(h)} \\
	& = \hat{W}^{(1:l)} \mathbf{x} + \hat{\mathbf{b}}^{(1:l)} \\
\end{aligned}
\end{equation}
where $\hat{W}^{(1:l)} = \prod_{h=0}^{l-2} \tilde{W}^{(l-h)} W^{(1)}$ is the coefficient matrix of $\mathbf{x}$, and $\hat{\mathbf{b}}^{(1:l)}$ is the sum of the remaining terms. 
The superscript ${(1:l)}$ indicates that $\hat{W}^{(1:l)} \mathbf{x} + \hat{\mathbf{b}}^{(1:l)}$ is equivalent to PLNN's forward propagation from layer $\mathcal{L}_1$ to layer $\mathcal{L}_l$.

\nop{
the PLNN's overall computation on $\mathbf{x}$ from layer $\mathcal{L}_1$ to layer $\mathcal{L}_l$. \todo{Here we have two variables, $\hat{W}^{(1:l)}$ and $\hat{\mathbf{b}}^{(1:l)}$. What is each meaning? What do you mean by ``overall computation''?}
}

Since the output of $\mathcal{N}$ on an input $\mathbf{x}\in\mathcal{X}$ is $F(\mathbf{x}) = \mathbf{a}^{(L)} = \text{\emph{softmax}}(\mathbf{z}^{(L)})$, the \textbf{closed form} of $F(\mathbf{x})$ is
\begin{equation}
\label{eqn:lineareqn}
	F(\mathbf{x}) = \text{\emph{softmax}}( \hat{W}^{(1:L-1)} \mathbf{x} + \hat{\mathbf{b}}^{(1:L-1)} )
\end{equation}

For a fixed PLNN $\mathcal{N}$ and a fixed instance $\mathbf{x}$, $\hat{W}^{(1:L-1)}$ and $\hat{\mathbf{b}}^{(1:L-1)}$ are constant parameters uniquely determined by the fixed configuration $\mathbf{C} = \text{\emph{conf}}(\mathbf{x})$.
Therefore, for a fixed input instance $\mathbf{x}$, $F(\mathbf{x})$ is a \textbf{linear classifier} whose decision boundary is explicitly defined by $\hat{W}^{(1:L-1)} \mathbf{x} + \hat{\mathbf{b}}^{(1:L-1)}$. 

Inspired by the interpretation method widely used by conventional linear classifiers, such as Logistic Regression and linear SVM~\cite{bishop2007pattern}, we interpret the prediction on a fixed instance $\mathbf{x}$ by the decision features of $F(\mathbf{x})$. Specifically, the entries of the $i$-th row of $\hat{W}^{(1:L-1)}$ are the \textbf{decision features} for the $i$-th class of instances.

\nop{[By regarding $F(\mathbf{x})$ as a linear classifier, we can explain the classification result of $\mathbf{x}$ by analyzing the important features that dominate the decision boundaries of $F(\mathbf{x})$. This is the same method widely used to explain the classification results of conventional linear classifiers, such as Logistic Regression and linear SVM~\cite{bishop2007pattern}.]}

\nop{
by analyzing the important features that dominate the decision boundaries of $F(\mathbf{x})$. 
This method is widely used to explain the classification results of conventional linear classifiers, such as Logistic Regression and linear SVM~\cite{bishop2007pattern}.
}

Equation~\ref{eqn:lineareqn} provides a straightforward way to interpret the classification result of a fixed instance.
However, individually interpreting the classification result of every single instance is far from the understanding of the overall behavior of a PLNN $\mathcal{N}$. 
\nop{
This is because, if we regard $\mathbf{x}$ as a variable that is not fixed, then $\hat{W}^{(1:L-1)}$ and $\hat{\mathbf{b}}^{(1:L-1)}$ are variables that depend on $\mathbf{x}$ in a non-linear manner; thus, $F(\mathbf{x})$ is, overall, a non-linear function that is hard to understand.
}
Next, we describe how to interpret the overall behavior of $\mathcal{N}$ by computing an interpretation model $\mathcal{M}$ that is mathematically equivalent to $\mathcal{N}$.

\subsection{Exact Interpretation of a PLNN}
\label{sec:eiap}
A fixed PLNN $\mathcal{N}$ with $N$ hidden neurons has at most $k^N$ configurations.
We represent the $h$-th configuration by $\mathbf{C}_h \in \mathcal{C}$, where $\mathcal{C}\subseteq \{1, \ldots, k\}^N$ is the set of all configurations of $\mathcal{N}$.

Recall that each instance $\mathbf{x}\in\mathcal{X}$ uniquely determines a configuration $\text{\emph{conf}}(\mathbf{x})\in \mathcal{C}$. 
Since the volume of $\mathcal{C}$, denoted by $|\mathcal{C}|$, is at most $k^N$, but the number of instances in $\mathcal{X}$ can be arbitrarily large, it is clear that at least one configuration in $\mathcal{C}$ should be shared by more than one instances in $\mathcal{X}$.

Denote by $P_h = \{\mathbf{x}\in \mathcal{X} \mid \text{\emph{conf}}(\mathbf{x}) = \mathbf{C}_h\}$ the set of instances that have the same configuration $\mathbf{C}_h$.
We prove in Theorem~\ref{thm:polytope} that for any configuration $\mathbf{C}_h \in \mathcal{C}$, $P_h$ is a convex polytope in $\mathcal{X}$.
\nop{[$P_h$ is a convex polytope in $\mathcal{X}$ for any configuration $\mathbf{c}(h) \in \mathcal{C}$.]}
\begin{theorem}
\label{thm:polytope}
Given a fixed PLNN $\mathcal{N}$ with $N$ hidden neurons, $\forall\mathbf{C}_h\in \mathcal{C}$, $P_h = \{\mathbf{x}\in \mathcal{X} \mid \text{\emph{conf}}(\mathbf{x}) = \mathbf{C}_h\}$ is a convex polytope in $\mathcal{X}$.
\end{theorem}

\begin{proof}
We prove by showing that $\text{\emph{conf}}(\mathbf{x}) = \mathbf{C}_h$ is equivalent to a finite set of linear inequalities with respect to $\mathbf{x}$.

When $l = 2$, we have $\mathbf{z}^{(2)} = W^{(1)}\mathbf{x} + \mathbf{b}^{(1)}$.
For $l\in\{3, \dots, L-1\}$, it follows Equation~\ref{eqn:simpleform} that $\mathbf{z}^{(l)} = \hat{W}^{(1:l-1)} \mathbf{x} + \hat{\mathbf{b}}^{(1:l-1)}$, which is a linear function of $\mathbf{x}$, because $\hat{W}^{(1:l)}$ and $\hat{\mathbf{b}}^{(1:l)}$ are constant parameters when $\mathbf{C}_h$ is fixed. 
In summary, given a fixed $\mathbf{C}_h$, $\mathbf{z}^{(l)}$ is a linear function of $\mathbf{x}$ for all $l\in\{2, \ldots, L-1\}$.

We show that $P_h$ is a convex polytope by showing that $\text{\emph{conf}}(\mathbf{x}) = \mathbf{C}_h$ is equivalent to a set of $2N$ linear inequalities with respect to $\mathbf{x}$.
Recall that $\mathbf{z}_i^{(l)}\in I_q$ if and only if $\mathbf{c}_i^{(l)} = q$ ($q\in\{1,\ldots,k\}$). 
Denote by $\psi: \{1, \ldots, k\} \rightarrow \{I_1, \ldots, I_k\}$ the bijective function that maps a configuration $\mathbf{c}_i^{(l)}$ to a real interval in $\{I_1, \ldots, I_k\}$, such that $\psi(\mathbf{c}_i^{(l)}) = I_q$ if and only if $\mathbf{c}_i^{(l)} = q$ ($q\in\{1,\ldots,k\}$).
Then, $\text{\emph{conf}}(\mathbf{x}) = \mathbf{C}_h$ is equivalent to a set of constraints, denoted by $Q_h = \{\mathbf{z}_i^{(l)} \in \psi(\mathbf{c}_i^{(l)}) \mid i\in\{1, \ldots, n_l\}, l\in\{2, \ldots, L-1\}\}$.
Since $\mathbf{z}_i^{(l)}$ is a linear function of $\mathbf{x}$ and $\psi(\mathbf{c}_i^{(l)})$ is a real interval, each constraint $\mathbf{z}_i^{(l)} \in \psi(\mathbf{c}_i^{(l)})$ in $Q_h$ is equivalent to two linear inequalities with respect to $\mathbf{x}$.
Therefore, $\text{\emph{conf}}(\mathbf{x}) = \mathbf{C}_h$ is equivalent to a set of $2N$ linear inequalities, which means $P_h$ is a convex polytope.
\end{proof}

According to Theorem~\ref{thm:polytope}, all instances sharing the same configuration $\mathbf{C}_h$ form a unique convex polytope $P_h$ that is explicitly defined by $2N$ linear inequalities in $Q_h$. 
Since $\mathbf{C}_h$ also determines the linear classifier for a fixed instance in Equation~\ref{eqn:lineareqn}, all instances in the same convex polytope $P_h$ share the same linear classifier determined by $\mathbf{C}_h$. 

\nop{
Denote by $F_h(\cdot)$ the linear classifier that is shared by all instances in $P_h$, 
we can rewrite $F(\cdot)$ in the following piecewise closed form.
\begin{equation}
\label{eqn:finalform}
F(\mathbf{x}) = 
    \left \{
        \begin{aligned}
            & F_1(\mathbf{x}), \quad\, \text{if } \mathbf{x}\in P_1 \cr
            & F_2(\mathbf{x}), \quad\, \text{if } \mathbf{x}\in P_2 \cr
            & \quad \quad \quad \vdots \cr
            & F_{|\mathcal{C}|}(\mathbf{x}), \quad \text{if } \mathbf{x}\in P_{|\mathcal{C}|} \cr
        \end{aligned}
    \right.
\end{equation}
where $P_1\cup \ldots \cup P_{|\mathcal{C}|} = \mathcal{X}$ and $\forall h\neq q, P_h \cap P_q = \emptyset$.

According to Equation~\ref{eqn:finalform}, }

\begin{algorithm}[t]
	\caption{$OpenBox(\mathcal{N}, D_{\text{train}})$}
	\label{alg:openbox}
	\KwIn{$\mathcal{N}\coloneqq$ a fixed PLNN, $D_{\text{train}}\subset \mathcal{X}$ the set of training instances used to train $\mathcal{N}$.}
	\KwOut{$\mathcal{M}\coloneqq$ a set of active LLCs}
	\BlankLine
	\begin{algorithmic}[1]
		\STATE Initialization: $\mathcal{M} = \emptyset$, $\mathcal{C} = \emptyset$.
		\FOR{each $\mathbf{x}\in D_{\text{train}}$}
			\STATE Compute the configuration by $\mathbf{C}_h \leftarrow \text{\emph{conf}}(\mathbf{x})$.
			\IF{$\mathbf{C}_h\not\in \mathcal{C}$}
				\STATE $\mathcal{C} \leftarrow \mathcal{C} \cup \mathbf{C}_h$ and $\mathcal{M} \leftarrow \mathcal{M} \cup (F_h(\mathbf{x}), P_h)$.
				\nop{\STATE Compute the closed form of $F_h(\mathbf{x})$ and $P_h$.}
				\nop{\STATE $\mathcal{M} \leftarrow \mathcal{M} \cup (F_h(\mathbf{x}), P_h)$.}
			\ENDIF
		\ENDFOR
		\RETURN $\mathcal{M}$.
	\end{algorithmic}
\end{algorithm}

Denote by $F_h(\cdot)$ the linear classifier that is shared by all instances in $P_h$, 
we can interpret $\mathcal{N}$ as a set of \textbf{local linear classifiers} (\textbf{LLCs}), each LLC being a linear classifier $F_h(\cdot)$ that applies to all instances in a convex polytope $P_h$.
Denote by a tuple $(F_h(\cdot), P_h)$ the $h$-th LLC, 
a fixed PLNN $\mathcal{N}$ is equivalent to a set of LLCs, denoted by $\mathcal{M} = \{(F_h(\cdot), P_h) \mid \mathbf{C}_h\in \mathcal{C}\}$.
We use $\mathcal{M}$ as our final interpretation model for $\mathcal{N}$.

For a fixed PLNN $\mathcal{N}$, if the states of the $N$ hidden neurons are independent, the PLNN $\mathcal{N}$ has $k^N$ configurations, which means $\mathcal{M}$ contains $k^N$ LLCs. 
However, due to the hierarchical structure of a PLNN, the states of a hidden neuron in $\mathcal{L}_l$ strongly correlate with the states of the neurons in the former layers $\mathcal{L}_q (q<l)$.
Therefore, the volume of $\mathcal{C}$ is much less than $k^N$, and the number of local linear classifiers in $\mathcal{M}$ is much less than $k^N$.
We discuss this phenomenon later in Table~\ref{Table:netstruct} and Section~\ref{sec:apleu}.

\nop{
Since a fixed PLNN $\mathcal{N}$ with $N$ hidden neurons can have at most $k^N$ configurations, $\mathcal{M}$ may contain at most $k^N$ LLCs. 
}

\nop{Thus, the number of local linear classifiers in $\mathcal{M}$ is much less than $k^N$. We discuss this phenomenon later in Section. xxx.}

\nop{In fact, the actual number of configurations in $\mathcal{C}$ is much less than $k^N$ due to the \todo{hierarchical} \mc{[deep]} structure of a PLNN. We will empirically discuss this phenomenon in Section. xxx.} \nop{[However, due to the deep structure of a PLNN, the states of the hidden neurons in $\mathcal{L}_l$ highly depends on the states of the neurons in the former layers of $\mathcal{L}_l$. Therefore, the actual number of configurations in $\mathcal{C}$ is much less than $k^N$. Thus, the number of local linear classifiers in $\mathcal{M}$ is much less than $k^N$. We discuss this phenomenon later in Section. xxx.]}

In practice, we do not need to compute the entire set of LLCs in $\mathcal{M}$ all at once.
Instead, we can first compute an active subset of $\mathcal{M}$, that is, the set of LLCs that are actually used to classify the available set of instances. Then, we can update $\mathcal{M}$ whenever a new LLC is used to classify a newly coming instance.

Algorithm~\ref{alg:openbox} summarizes the $OpenBox$ method, which computes $\mathcal{M}$ as the active set of LLCs that are actually used to classify the set of training instances, denoted by $D_{\text{train}}$.

\rev{
The time cost of Algorithm~\ref{alg:openbox} consists of the time $T_{conf}$ to compute $conf(\mathbf{x})$ in step 3 and the time $T_\text{LLC}$ to compute the LLC $(F_h(\mathbf{x}), P_h)$ in step 5.
Since $T_{conf}$ and $T_\text{LLC}$ are dominated by matrix (vector) multiplications, we evaluate the time cost of Algorithm~\ref{alg:openbox} by the number of scalar multiplications.
First, since we compute $conf(\mathbf{x})$ by forward propagating from layer $\mathcal{L}_1$ to layer $\mathcal{L}_{L-1}$, $T_{conf}=\sum_{l=2}^{L-1} n_l n_{l-1}$.
Second, since $(F_h(\mathbf{x}), P_h)$ is determined by the set of tuples $\mathcal{G}=\{(\hat{W}^{(1:l)}, \hat{\mathbf{b}}^{(1:l)}) \mid l\in\{1, \ldots, L-1\}\}$, $T_\text{LLC}$ is the time to compute $\mathcal{G}$.
Given $(\hat{W}^{(1:l-1)}, \hat{\mathbf{b}}^{(1:l-1)})$, we can compute $(\hat{W}^{(1:l)}, \hat{\mathbf{b}}^{(1:l)})$ by plugging $\mathbf{z}^{(l)} = \hat{W}^{(1:l-1)} \mathbf{x} + \hat{\mathbf{b}}^{(1:l-1)}$ (Equation~\ref{eqn:simpleform}) into Equation~\ref{eqn:genterm}, and the time cost is $n_{l+1}n_{l}(n_1+1)$. Since $\hat{W}^{(1:1)} = W^{(1)}$ and $\hat{\mathbf{b}}^{(1:1)}=\mathbf{b}^{(1)}$, we can iteratively compute $\mathcal{G}$. The overall time cost is $T_\text{LLC} = \sum_{l=2}^{L-1} n_{l+1}n_{l}(n_1+1)$.
}

\rev{
The worst case of Algorithm~\ref{alg:openbox} happens when every instance $\mathbf{x}\in D_\text{train}$ has a unique configuration $conf(\mathbf{x})$. 
Denote by $|D_\text{train}|$ the number of training instances, the time cost of Algorithm~\ref{alg:openbox} in the worst case is $|D_\text{train}| (T_{conf} + T_{LLC})$.
Since $n_l, l\in\{2, \ldots, L-1\}$ are constants and $n_1 = d$ is the size of the input $\mathbf{x}\in \mathbb{R}^d$, the time complexity of Algorithm~\ref{alg:openbox} is $O(|D_\text{train}| d)$.
}

\nop{
As discussed later in Table~\ref{Table:netstruct} and Section~\ref{sec:apleu}, the actual number of unique configurations is quite small. Therefore, Algorithm~\ref{alg:openbox} runs very fast in practice.
}

Now, we are ready to introduce how to interpret the classification result of an instance $\mathbf{x}\in P_h, h\in\{1, \ldots, |\mathcal{C}|\}$.
First, we interpret the classification result of $\mathbf{x}$ using the decision features of $F_h(\mathbf{x})$ (Section~\ref{sec:eecrfi}).
Second, we interpret why $\mathbf{x}$ is contained in $P_h$ using the \textbf{polytope boundary features} (\textbf{PBFs}), which are the decision features of the polytope boundaries. More specifically, a polytope boundary of $P_h$ is defined by a linear inequality $\mathbf{z}_i^{(l)} \in \psi(\mathbf{c}_i^{(l)})$ in $Q_h$. 
By Equation~\ref{eqn:simpleform}, $\mathbf{z}_i^{(l)}$ is a linear function with respect to $\mathbf{x}$. 
The PBFs are the coefficients of $\mathbf{x}$ in $\mathbf{z}_i^{(l)}$.

We also discover that some linear inequalities in $Q_h$ are redundant whose hyperplanes do not intersect with $P_h$.
To simplify our interpretation on the polytope boundaries, we remove such redundant inequalities by Caron's method~\cite{caron1989degenerate} and focus on studying the PBFs of the non-redundant ones.

The advantages of $OpenBox$ are three-fold as follows.
First, our interpretation is exact, because the set of LLCs in $\mathcal{M}$ are mathematically equivalent to the classification function $F(\cdot)$ of $\mathcal{N}$.
Second, our interpretation is group-wise consistent. It is due to the reason that all instances in the same convex polytope are classified by exactly the same LLC, and thus the interpretations are consistent with respect to a given convex polytope.
\rev{Last, our interpretation is easy to compute due to the low time complexity of Algorithm~\ref{alg:openbox}.}


\section{Experiments}
\label{sec:exp}
In this section, we evaluate the performance of $OpenBox$, and compare it with the state-of-the-art method LIME~\cite{ribeiro2016should}.
In particular, we address the following questions:
(1) What are the LLCs look like?
(2) Are the interpretations produced by LIME and $OpenBox$ exact and consistent?
(3) Are the decision features of LLCs easy to understand, and can we improve the interpretability of these features by non-negative and sparse constraints?
(4) How to interpret the PBFs of LLCs?
(5) How effective are the interpretations of $OpenBox$ in hacking and debugging a PLNN model?

Table~\ref{Table:models} shows the details of the six models we used.
For both PLNN and PLNN-NS, we use the same network structure described in Table~\ref{Table:netstruct}, and adopt the widely used activation function: ReLU~\cite{glorot2011deep}.
We apply the non-negative and sparse constraints proposed by Chorowski~\textit{et al.}~\cite{chorowski2015learning} to train PLNN-NS.
Since our goal is to comprehensively study the interpretation effectiveness of $OpenBox$ rather than achieving state-of-the-art classification performance, we use relatively simple network structures for PLNN and PLNN-NS, which are still powerful enough to achieve significantly better classification performance than Logistic Regression (LR).
The decision features of LR, LR-F, LR-NS and LR-NSF are used as baselines to compare with the decision features of LLCs.

The Python code of LIME is published by its authors\footnote{\url{https://github.com/marcotcr/lime}}.
The other methods and models are implemented in Matlab. PLNN and PLNN-NS are trained using the DeepLearnToolBox~\cite{IMM2012-06284}.
All experiments are conducted on a PC with a Core-i7-3370 CPU (3.40 GHz), 16GB main memory, and a 5,400 rpm hard drive running Windows 7 OS.

We use the following data sets. Detailed information of the data sets is shown in Table~\ref{Table:datasets}.

\textbf{Synthetic (SYN) Data Set.}
As shown in Figure~\ref{Fig:exp1}(a), this data set contains 20,000 instances uniformly sampled from a quadrangle in 2-dimensional Euclidean space.
The red and blue points are positive and negative instances, respectively.
\rev{Since we only use SYN to visualize the LLCs of a PLNN and we do not perform testing on SYN, we use all instances in SYN as the training data.}

\nop{We use all instances in SYN as training data to visualize the LLCs of a PLNN.}

\begin{table}[t]
\centering\small
\caption{The models to interpret. LR is Logistic Regression. NS means non-negative and sparse constraints. Flip means the model is trained on the instances with flipped labels.}
\label{Table:models}
\vspace{-2mm}
\begin{tabular}{|p{15mm}<\centering|c|c|c|c|c|c|}
\hline

Models			&	PLNN	& 	PLNN-NS		&	LR		&	LR-F			&	LR-NS		&	LR-NSF	\\ \hline
NS				&	$\times$	&	$\checkmark$	&	$\times$	&	$\times$		&	$\checkmark$	&	$\checkmark$	\\ \hline
Flip 		 		&	$\times$	&	$\times$		&	$\times$	&	$\checkmark$	&	$\times$		&	$\checkmark$	\\ \hline

\end{tabular}
\end{table}

\begin{table}[t]
\centering\small
\caption{The network structures $(n_1, n_2, \ldots, n_L)$ and the number of configurations $|\mathcal{C}|$ of PLNN and PLNN-NS.
The neurons in successive layers are initialized to be fully connected. 
$k=2$ is the number of linear functions of ReLU, $N$ is the number of hidden neurons.}
\label{Table:netstruct}
\vspace{-2mm}
\begin{tabular}{|m{12mm}<\centering|m{20mm}<\centering|c|c|c|c|}
\hline

\multirow{2}{*}{Data Sets}	&	\multirow{2}{20mm}{\centering{\# Neurons $(n_1, n_2, \ldots, n_L)$}}	&	\multicolumn{2}{c|}{PLNN}	&	\multicolumn{2}{c|}{PLNN-NS} 	\\ \cline{3-6}	

			&			&	$|\mathcal{C}|$		&	$k^N$	&	$|\mathcal{C}|$		&	$k^N$		\\ \hline

SYN			&	$(2, 4, 16, 2, 2)$	&	$266$	&	$2^{22}$	&	$41$		&	$2^{22}$		 		\\ \hline

FMNIST-1		&	$(784, 8, 2, 2)$		&	$78$		&	$2^{10}$	&	$3$		&	$2^{10}$				\\ \hline

FMNIST-2		&	$(784, 8, 2, 2)$		&	$23$		&	$2^{10}$	&	$18$		&	$2^{10}$				\\ \hline

\end{tabular}
\end{table}

\begin{table}[t]
\centering\small
\caption{Detailed description of data sets.}
\label{Table:datasets}
\vspace{-2mm}
\begin{tabular}{|c|p{13mm}<\centering|p{13mm}<\centering|p{13mm}<\centering|p{13mm}<\centering|}
\hline
\multirow{2}{*}{Data Sets}  & \multicolumn{2}{c|}{Training Data} & \multicolumn{2}{c|}{Testing Data} 		\\ \cline{2-5}

		     	& 	\# Positive		& 	\# Negative	&	\# Positive		&	\# Negative		\\ \hline

SYN			&	6,961		&	13,039		&	N/A			&	N/A				\\ \hline

FMNIST-1		&	4,000		&	4,000		& 	3,000		&	3,000			\\ \hline

FMNIST-2		&	4,000		&	4,000		& 	3,000		&	3,000			\\ \hline

\nop{LMR			& 	12,500		& 	12,500		& 	12,500		& 	12,500			\\ \hline}

\end{tabular}
\end{table}

\newcommand{\parawidthExpOne}{42mm}
\begin{figure}[t]
\centering
\subfigure[training data of SYN]{\includegraphics[width=\parawidthExpOne]{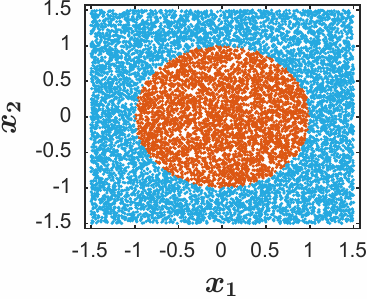}}
\subfigure[prediction results of PLNN]{\includegraphics[width=\parawidthExpOne]{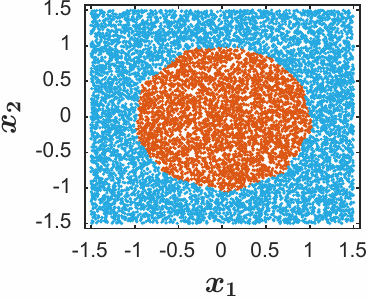}}

\subfigure[convex polytopes]{\includegraphics[width=\parawidthExpOne]{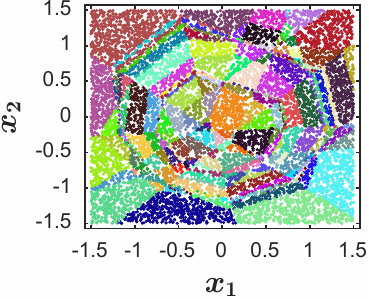}}
\subfigure[LLCs]{\includegraphics[width=\parawidthExpOne]{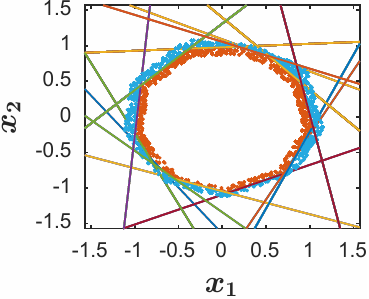}}

\caption{The LLCs of the PLNN trained on SYN.}
\label{Fig:exp1}
\end{figure}

\textbf{FMNIST-1 and FMNIST-2 Data Sets}.
Each of these data sets contains two classes of images in the Fashion MNIST data set~\cite{xiao2017/online}.
FMNIST-1 consists of the images of \textit{Ankle Boot} and \textit{Bag}.
FMNIST-2 consists of the images of \textit{Coat} and \textit{Pullover}.
All images in FMNIST-1 and FMNIST-2 are 28-by-28 grayscale images.
We represent an image by cascading the 784 pixel values into a 784-dimensional feature vector.
The Fashion MNIST data set is available online\footnote{\url{https://github.com/zalandoresearch/fashion-mnist}}.

\nop{
\textbf{Large Movie Review (LMR) data set~\cite{maas2011learning}.}
This is a benchmark text dataset for binary sentiment classification of movie reviews. 
It includes 25,000 highly polar movie reviews for training, 25,000 for testing, and a vocabulary that contains 89,527 words.
We use the vocabulary to parse each movie review into a 89,527-dimensional bag-of-word vector, where the $i$-th dimension is the tf-idf of the $i$-th word in the vocabulary.
}
\nop{This data set is available at ``\url{http://ai.stanford.edu/~amaas/data/sentiment/}''.}

\nop{
For the purpose of generality, we use DNN [xxx] as our target model to interpret, and adopt the famous piecewise linear function ReLU~\cite{glorot2011deep} as the activation function of hidden neurons.
}

\nop{
This data set consists of Zalando's article images. 
It contains a training set of 60,000 grayscale images and a testing set of 10,000 grayscale images. 
The pixel resolution of every image is 28x28. Each image is associated with a label from ten classes.
We build two data sets for binary classification by pairing up four classes of the FMNIST data set.
}

\subsection{What Are the LLCs Look Like?}
We demonstrate our claim in Theorem~\ref{thm:polytope} by visualizing the LLCs of the PLNN trained on SYN. 

Figures~\ref{Fig:exp1}(a)-(b) show the training instances of SYN and the prediction results of PLNN on the training instances, respectively. 
Since all instances are used for training, the prediction accuracy is 99.9\%.

In Figure~\ref{Fig:exp1}(c), we plot all instances with the same configuration in the same colour. Clearly, all instances with the same configuration are contained in the same convex polytope.
This demonstrates our claim in Theorem~\ref{thm:polytope}.

Figure~\ref{Fig:exp1}(d) shows the LLCs whose convex polytopes cover the decision boundary of PLNN and contain both positive and negative instances.
As it is shown, the solid lines show the decision boundaries of the LLCs, which capture the difference between positive and negative instances, and form the overall decision boundary of PLNN.
A convex polytope that does not cover the boundary of PLNN contains a single class of instances. 
The LLCs of these convex polytopes capture the common features of the corresponding class of instances.
As to be analyzed in the following subsections, the set of LLCs produce exactly the same prediction as PLNN, and also capture meaningful decision features that are easy to understand.

\nop{In summary, This demonstrates our claim in Theorem~\ref{thm:polytope}.}

\nop{The decision features of these LLCs capture the difference between the positive and negative instances.}

\nop{
We can observe from Figure~\ref{Fig:exp1}(a)-(c) that the convex polytopes on the decision boundary of PLNN contain both positive and negative instances. 

The LLCs of these convex polytopes are shown in Figure~\ref{Fig:exp1}(d), where each line is a linear classifier that classifies the instances within a convex polytope.
}

\nop{
As illustrated later in \todo{Section~xxx and Section~xxx}, the LLCs produce exactly the same prediction results as PLNN, and the boundaries of the convex polytopes capture interpretable common features of all instances it contains.
}

\nop{
Figure~\ref{Fig:exp1} (d) shows the LLCs of the convex polytopes on the decision boundary, each line (d) is a linear classifier that classifies the instances within a convex polytope.

that form the overall decision boundary of PLNN. 
Each line in (d) is a linear classifier that classifies the instances within a convex polytope.

As shown in Figure~\ref{Fig:exp1} (d), the LLCs of the convex polytopes on the decision boundary distinguish the positive and negative classes and form the overall decision boundary of PLNN.
}

\subsection{Are the Interpretations Exact and Consistent?}
\label{sec:exact_consist}
Exact and consistent interpretations are naturally favored by human minds.
In this subsection, we systematically study the exactness and consistency of the interpretations of LIME and $OpenBox$ on FMNIST-1 and FMNIST-2.
Since LIME is too slow to process all instances in 24 hours,
for each of FMNIST-1 and FMNIST-2, we uniformly sample 600 instances from the testing set, and conduct the following experiments on the sampled instances.

\nop{
we uniformly sampled two sets instances, each from the testing sets of FMNIST-1 and FMNIST-2, respectively.

for each of FMNIST-1 and FMNIST-2, we conduct the following experiments on a subset of 600 instances uniformly sampled from the testing set.
}

We first analyze the \textbf{exactness of interpretation} by comparing the predictions computed by the local interpretable model of LIME, the LLCs of $OpenBox$ and  PLNN, respectively.
The prediction of an instance is the probability of classifying it as a positive instance.

\newcommand{\parawidthExpTwo}{42mm}
\begin{figure}[t]
\centering
\subfigure[FMNIST-1]{\includegraphics[width=\parawidthExpTwo]{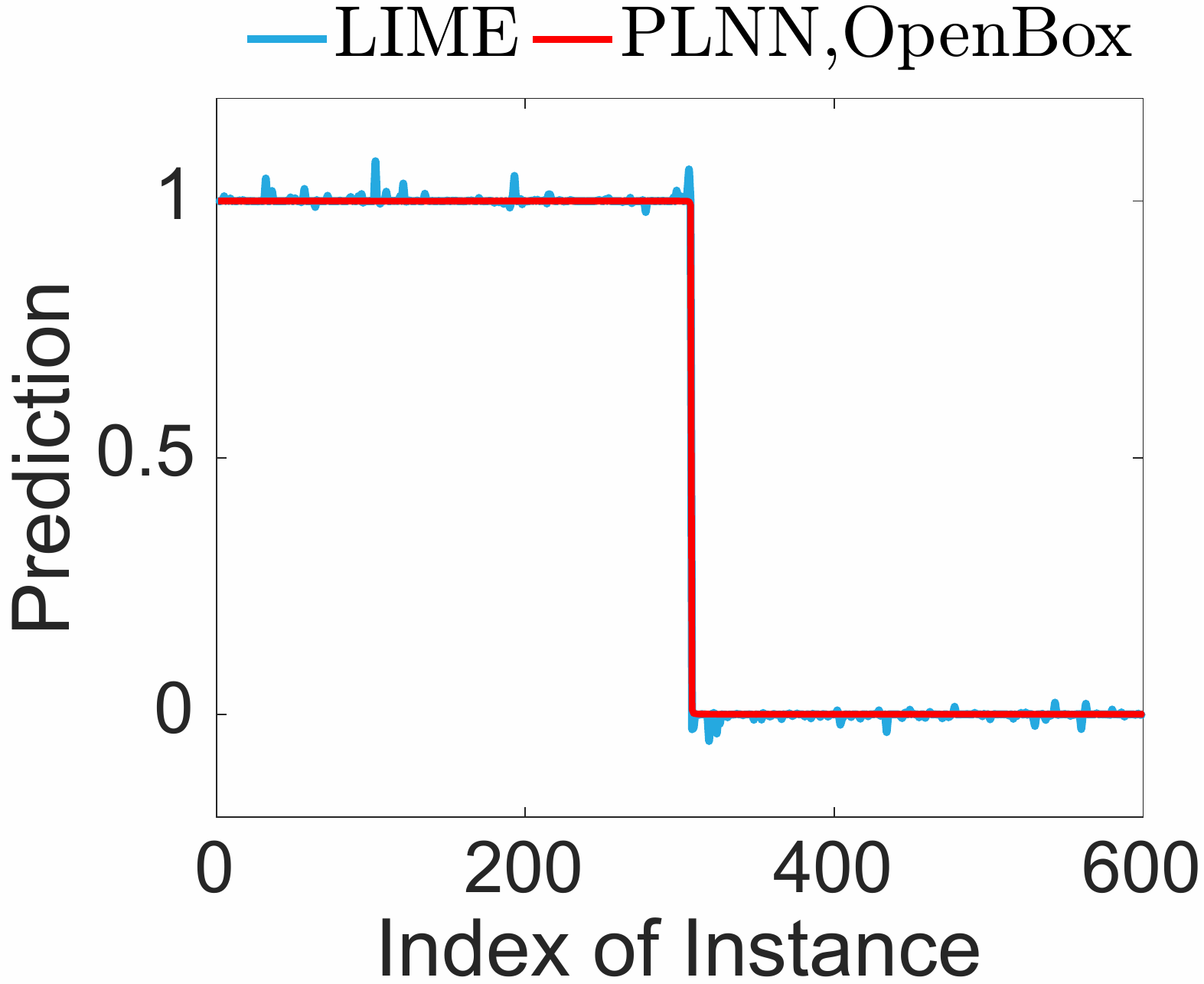}}
\subfigure[FMNIST-2]{\includegraphics[width=\parawidthExpTwo]{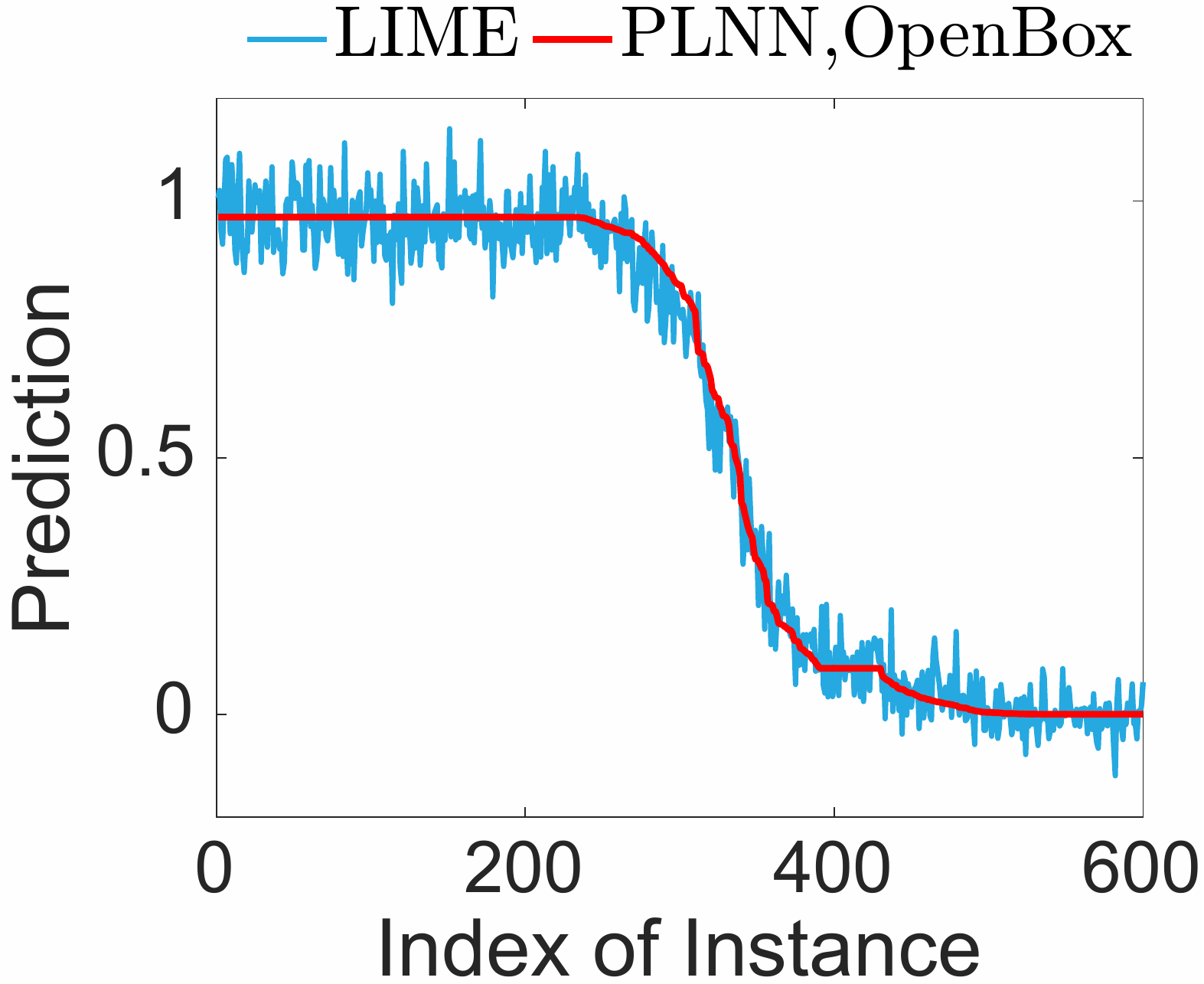}}
\caption{The predictions of LIME, OpenBox and PLNN. \rev{The predictions of all methods are computed individually and independently.} We sort the results by PLNN's predictions in descending order.
}
\label{Fig:exp2}
\end{figure}

\newcommand{\parawidthExpThree}{42mm}
\begin{figure}[t]
\centering
\subfigure[FMNIST-1]{\includegraphics[width=\parawidthExpThree]{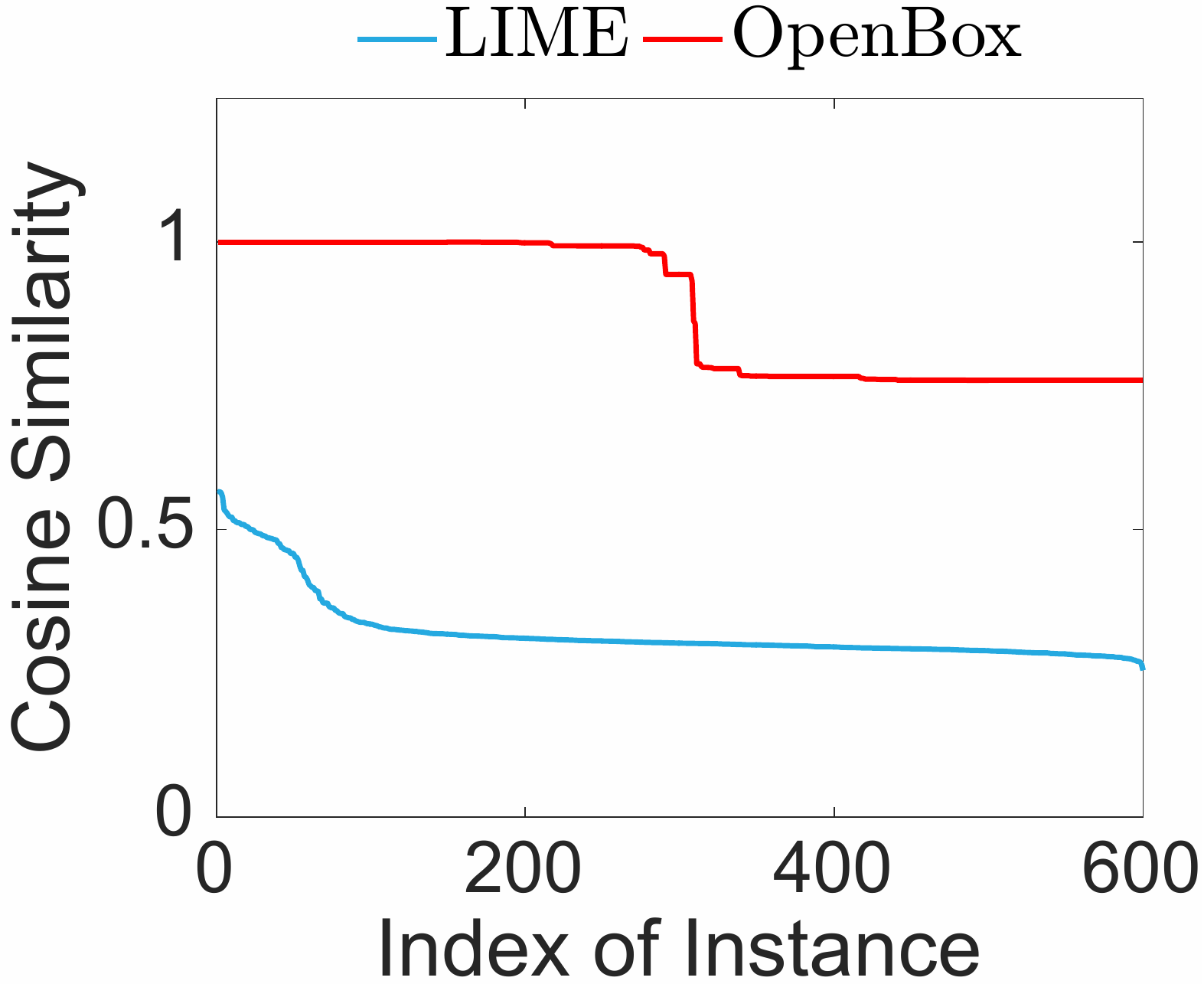}}
\subfigure[FMNIST-2]{\includegraphics[width=\parawidthExpThree]{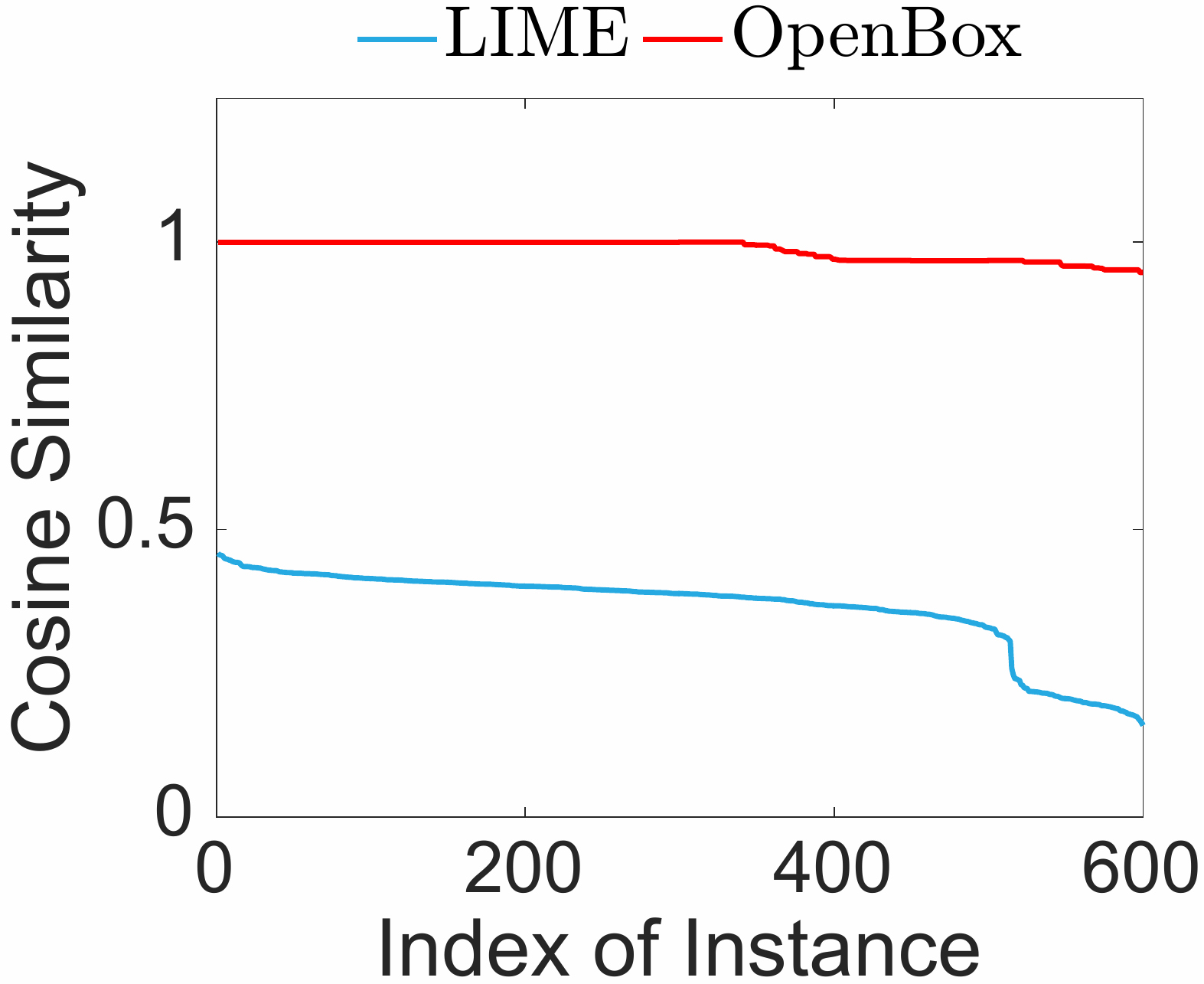}}
\caption{The cosine similarity between the decision features of each instance and its nearest neighbour. The results of LIME and $OpenBox$ are separately sorted by cosine similarity in descending order.}
\label{Fig:exp3}
\end{figure}

In Figure~\ref{Fig:exp2}, since LIME does not guarantee zero approximation error on the local predictions of PLNN, 
the predictions of LIME are not exactly the same as PLNN on FMNIST-1, and are dramatically different from PLNN on FMNIST-2.
The difference of predictions is more significant on FMNIST-2, because the images in FMNIST-2 are more difficult to distinguish, which makes the decision boundary of PLNN more complicated and harder to approximate.
We can also see that the predictions of LIME exceed $[0, 1]$. This is because the output of the interpretable model of LIME is not a probability at all.
As a result, it is arguable that the interpretations computed by LIME may not truthfully describe the exact behavior of PLNN.
In contrast, since the set of LLCs computed by $OpenBox$ is mathematically equivalent to $F(\cdot)$ of PLNN, the predictions of $OpenBox$ are exactly the same as PLNN on all instances.
Therefore, the decision features of LLCs exactly describe the overall behavior of PLNN.

\nop{
there is no guarantee that the interpretations computed by LIME exactly describe the behavior of PLNN.
}

\nop{Since the images in FMNIST-2 are more difficult to separate than the images in FMNIST-1, the difference between the predictions of LIME and PLNN is more significant on FMNIST-2.}

Next, we study the \textbf{interpretation consistency} of LIME and $OpenBox$ by analyzing the similarity between the interpretations of similar instances.

In general, a consistent interpretation method should provide similar interpretations for similar instances.
For an instance $\mathbf{x}$, denote by $\mathbf{x}'$ the nearest neighbor of $\mathbf{x}$ by Euclidean distance, by $\gamma, \gamma' \in\mathbb{R}^d$ the decision features for the classification of $\mathbf{x}$ and $\mathbf{x}'$, respectively.
We measure the consistency of interpretation by the cosine similarity between $\gamma$ and $\gamma'$, where a larger cosine similarity indicates a better interpretation consistency.

\nop{
For LIME and $OpenBox$, the vectors of feature weights are the coefficients that define the decision boundaries of the local interpretable model~\cite{ribeiro2016should} and the set of LLCs, respectively.
}

As shown in Figure~\ref{Fig:exp3}, the cosine similarity of $OpenBox$ is equal to 1 on about 50\% of the instances, because $OpenBox$ consistently gives the same interpretation for all instances in the same convex polytope.
Since the nearest neighbours $\mathbf{x}$ and $\mathbf{x}'$ may not belong to the same convex polytope, the cosine similarity of $OpenBox$ is not always equal to 1 on all instances.
In constrast, since LIME computes individual interpretation based on the unique local perturbations of every single instance, the cosine similarity of LIME is significantly lower than $OpenBox$ on all instances. This demonstrates the superior interpretation consistency of $OpenBox$.

In summary, the interpretations of $OpenBox$ are exact, and are much more consistent than the interpretations of LIME.

\nop{This is because $OpenBox$ consistently gives the same interpretation for all instances within the same convex polytope, }


\begin{table}[t]
\centering
\caption{The training and testing accuracy of all models.}
\label{Table:exp4_acu}
\vspace{-2mm}\small
\begin{tabular}{|c|p{11mm}<\centering|p{11mm}<\centering|p{11mm}<\centering|p{11mm}<\centering|}
    \hline
    Data Set      			& \multicolumn{2}{c|}{FMNIST-1} & \multicolumn{2}{c|}{FMNIST-2} \\ \hline
    Accuracy       		& 	Train 		& 	Test	 		& 	Train	   		& 	Test 				\\ \hline
    LR  				& 	0.998  		& 	0.997 		& 	0.847        	& 	0.839	        		\\ \hline
    LR-F  				& 	0.998  		& 	0.997 		& 	0.847        	& 	0.839	        		\\ \hline
    PLNN	     			& 	\textbf{1.000}  	& 	\textbf{0.999}	& 	\textbf{0.907}    & 	\textbf{0.868} 		 \\ \hline \hline
    
    LR-NS 				& 	0.772		& 	0.776		& 	0.711        		& 	0.698 	   		\\ \hline
    LR-NSF		  	& 	0.989  		& 	0.989		& 	0.782	 	& 	0.791        		\\ \hline
    PLNN-NS   			& 	\textbf{1.000}  	& 	\textbf{0.999}	& 	\textbf{0.894}    &	\textbf{0.867}    		\\ \hline
\end{tabular}
\end{table}

\nop{
    LR   		& 	1.00  	& 	1.00		& 	0.85          & 	0.84       	 \\ \hline
    PLNN	     	& 	1.00 		& 	1.00		& 	\textbf{0.91}  	&	\textbf{0.87}	 	 \\ \hline \hline
    
    LR-NS	 			& 	0.77  	& 	0.78		& 	0.71        	& 	0.70 	  	 \\ \hline
    LR-NS (reversed)  	& 	0.99  	& 	0.99		& 	0.78 		& 	0.79       	 \\ \hline
    PLNN-NS   			& 	\textbf{1.00}  	& 	\textbf{1.00}	& 	\textbf{0.89}          & 	\textbf{0.87}     	\\ \hline
}

\subsection{Decision Features of LLCs and the Effect of Non-negative and Sparse Constraints}
Besides exactness and consistency, a good interpretation should also have a strong semantical meaning, such that the ``thoughts'' of an intelligent machine can be easily understood by a human brain.
In this subsection, we first show the meaning of the decision features of LLCs, then study the effect of the non-negative and sparse constraints in improving the interpretability of the decision features.
The decision features of PLNN and PLNN-NS are computed by $OpenBox$.
The decision features of LR, LR-F, LR-NS and LR-NSF are used as baselines.
Table~\ref{Table:exp4_acu} shows the accuracy of all models.

\nop{
To our surprise, the decision features of the LLCs have a strong semantical meaning that is as easy to understand as the decision features of the LR models, such as LR, LR-F, LR-NS and LR-NSF.
}
\nop{
We use the decision features of LR, LR-F, LR-NS and LR-NSF as baselines, and found that the 
}

\nop{
 to interpret the PLNN and PLNN-NS trained on FMNIST-1 and FMNIST-2.
}

\nop{
interpreting the PLNN and PLNN-NS using the decision features computed by $OpenBox$.

Table~\ref{Table:exp4_acu} shows the training and testing accuracy of all models on FMNIST-1 and FMNIST-2.
Clearly, both PLNN and PLNN-NS achieve the best accuracy on both data ses.
As shown in Table~\ref{Table:exp4_acu}, both PLNN and PLNN-NS achieve the best accuracy on FMNIST-1 and FMNIST-2.

What are the reasons for PLNN and PLNN-NS to achieve the best accuracy on both data sets?
Next, we apply $OpenBox$ study this problem.
}

Figure~\ref{Fig:decision_boundary_89} shows the decision features of all models on FMNIST-1.
Interestingly, the decision features of PLNN are as easy to understand as the decision features of LR and LR-F. All these features clearly highlight meaningful image parts, such as the ankle and heel of \emph{Ankle Boot}, and the upper left corner of \emph{Bag}.
A closer look at the the average images suggests that these decision features describe the difference between \emph{Ankle Boot} and \emph{Bag}.

The decision features of PLNN capture more detailed difference between \emph{Ankle Boot} and \emph{Bag} than the decision features of LR and LR-F. 
This is because the LLCs of PLNN only capture the difference between a subset of instances within a convex polytope, however, LR and LR-F capture the overall difference between all instances of \emph{Ankle Boot} and \emph{Bag}.
The accuracies of PLNN, LR and LR-F are comparable because the instances of \emph{Ankle Boot} and \emph{Bag} are easy to distinguish.
However, as to be shown in Figure~\ref{Fig:decision_boundary_24}, when the instances are hard to distinguish, PLNN captures much more detailed features than LR and LR-F, and achieves a significantly better accuracy.

Figure~\ref{Fig:decision_boundary_24} shows the decision features of all models on FMNIST-2.
As it is shown, LR and LR-F capture decision features with a strong semantical meaning, such as the collar and breast of \emph{Coat}, and the shoulder of \emph{Pullover}.
However, these features are too general to accurately distinguish between \emph{Coat} and \emph{Pullover}.
Therefore, LR and LR-F do not achieve a high accuracy.
Interestingly, the decision features of PLNN capture much more details than LR and LR-F, which leads to the superior accuracy of PLNN.

The superior accuracy of PLNN comes at the cost of cluttered decision features that may be hard to understand.
Fortunately, applying non-negative and sparse constraints on PLNN effectively improves the interpretability of the decision features without affecting the classification accuracy.

\nop{
According to the above results, the decision features of PLNN, LR and LR-F highlight similar semantical parts of the images.
This demonstrates the strong semantical meaning of the decision features computed by $OpenBox$.
}
\nop{
highlight similar parts of \emph{Coat} and \emph{Pullover} as the collar, breast and shoulder that are highlighted by the LR models.
}

\nop{
have a similar semantical meaning as the decision features of LR models, however, the decision features of PLNN captures more details. 
}
\nop{
much more detail than LR and LR-F, thus PLNN achieves the best accuracy.
}
\nop{
Since a PLNN is equivalent to a set of LLCs, each of which learns detailed differences between a small subset of instances, 
the overall description strength of a PLNN is much stronger than the single linear classifier of a LR model.
}

\newcommand{\parawidthExpFour}{15mm}
\newcommand{\parawidthExpFourII}{9.5mm}
\newcommand{\paraheightLegend}{15mm}
\begin{figure}[t]

\begin{minipage}[t]{1\linewidth}
\centering
\subfigure[Avg. Image]{\includegraphics[width=\parawidthExpFour]{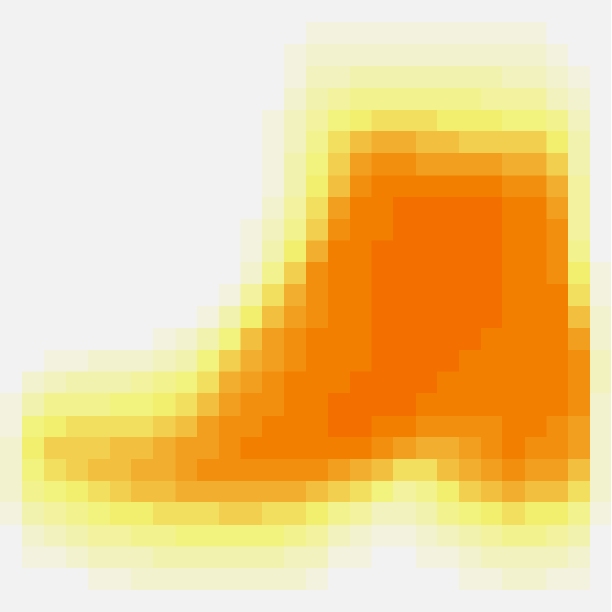}}
\subfigure[LR]{\includegraphics[width=\parawidthExpFour]{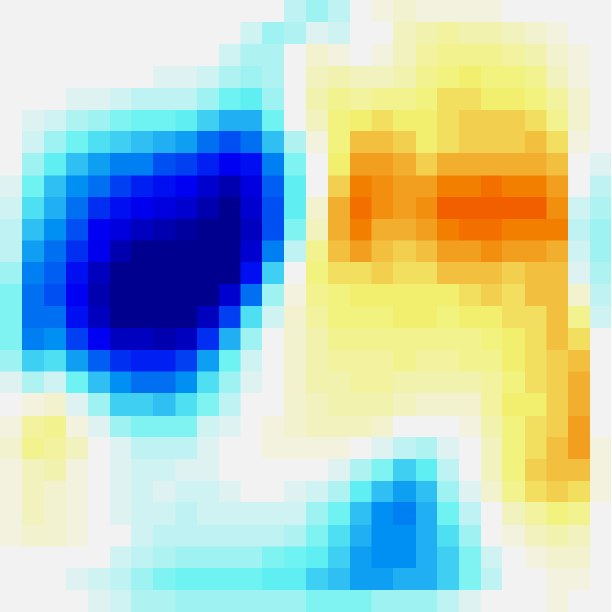}}
\subfigure[LR-NS]{\includegraphics[width=\parawidthExpFour]{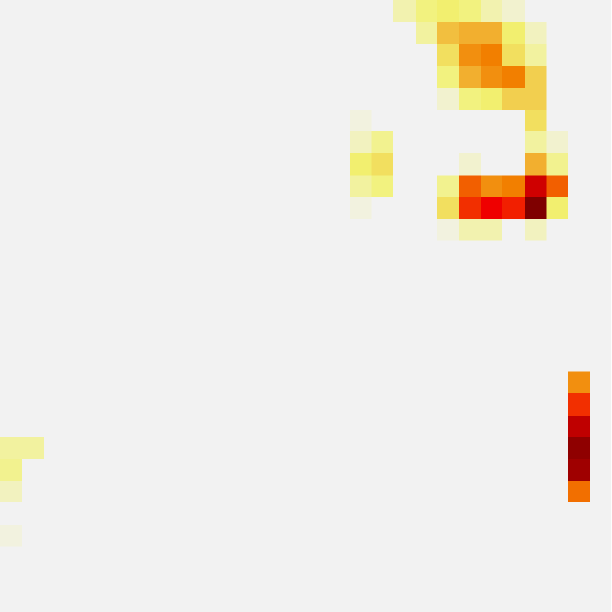}}
\subfigure[PLNN]{\includegraphics[width=\parawidthExpFour]{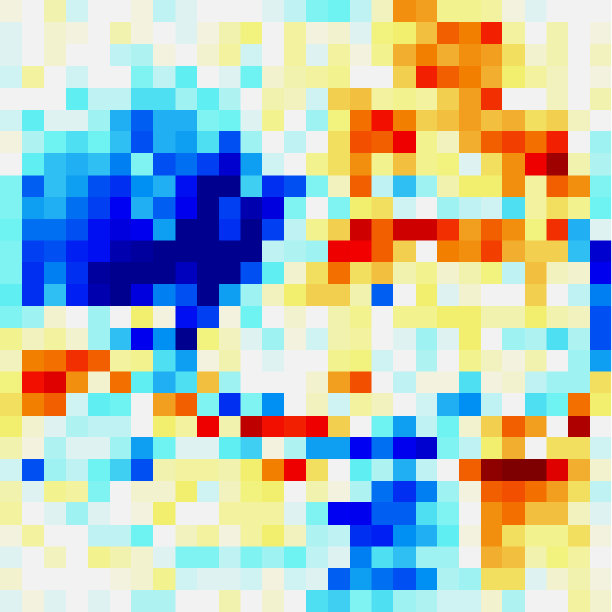}}
\subfigure[PLNN-NS]{\includegraphics[width=\parawidthExpFour]{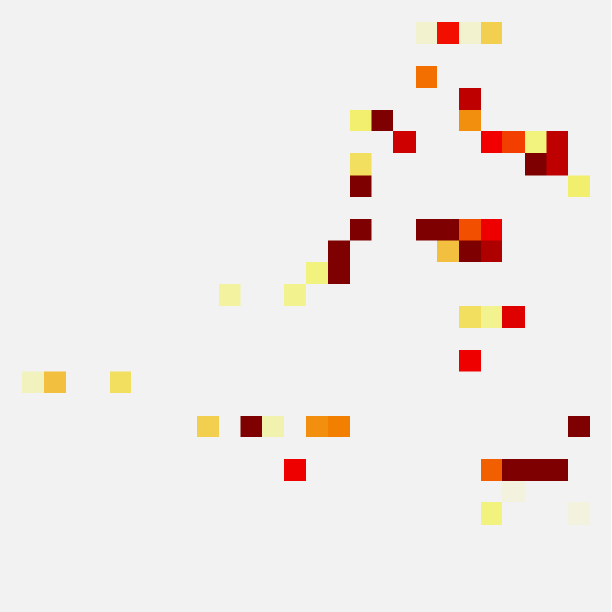}}
\includegraphics[height=\paraheightLegend]{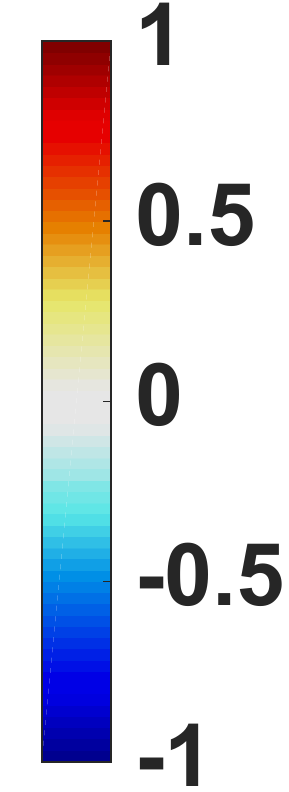}

\subfigure[Avg. Image]{\includegraphics[width=\parawidthExpFour]{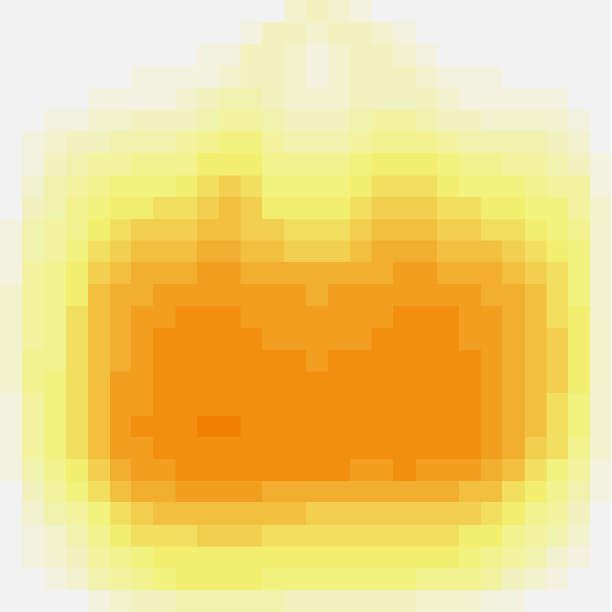}}
\subfigure[LR-F]{\includegraphics[width=\parawidthExpFour]{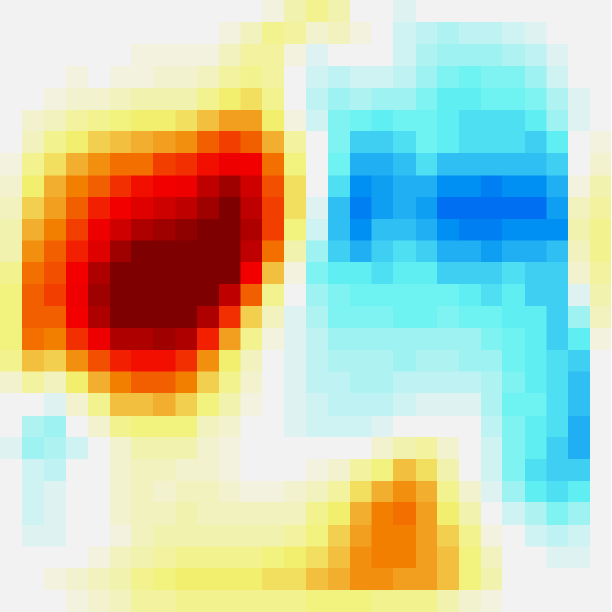}}
\subfigure[LR-NSF]{\includegraphics[width=\parawidthExpFour]{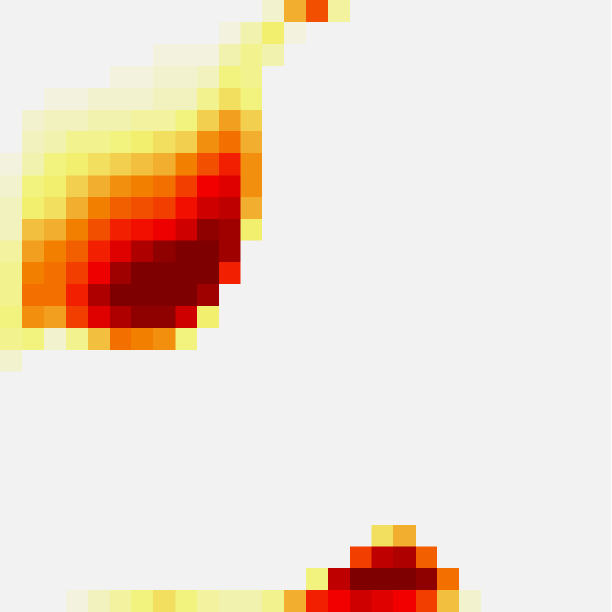}}
\subfigure[PLNN]{\includegraphics[width=\parawidthExpFour]{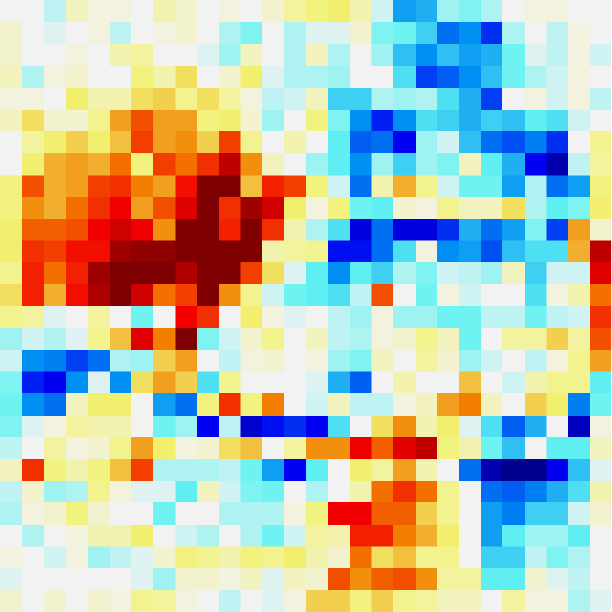}}
\subfigure[PLNN-NS]{\includegraphics[width=\parawidthExpFour]{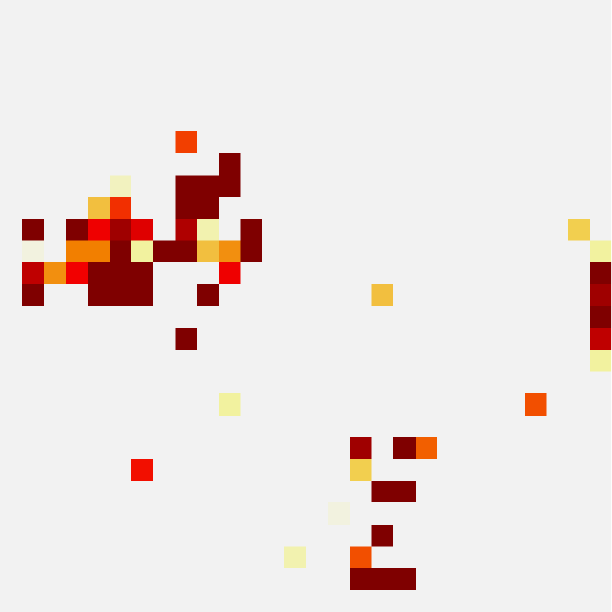}}
\includegraphics[height=\paraheightLegend]{Figures/legend}
\end{minipage}
\caption{The decision features of all models on FMNIST-1.
(a)-(e) and (f)-(j) show the average image and the decision features of all models for \emph{Ankle Boot} and \emph{Bag}, respectively.
For PLNN and PLNN-NS, we show the decision features of the LLC whose convex polytope contains the most instances.
}
\label{Fig:decision_boundary_89}
\end{figure}

\begin{figure}[t]
\centering
\begin{minipage}[t]{1\linewidth}
\centering
\subfigure[Avg. Image]{\includegraphics[width=\parawidthExpFour]{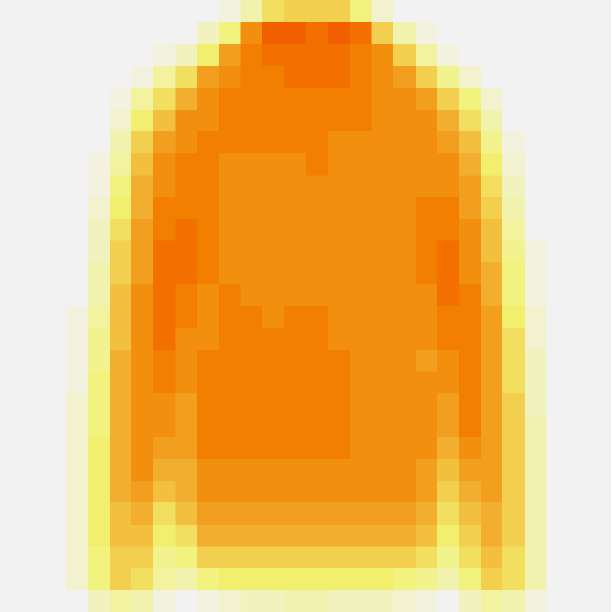}}
\subfigure[LR]{\includegraphics[width=\parawidthExpFour]{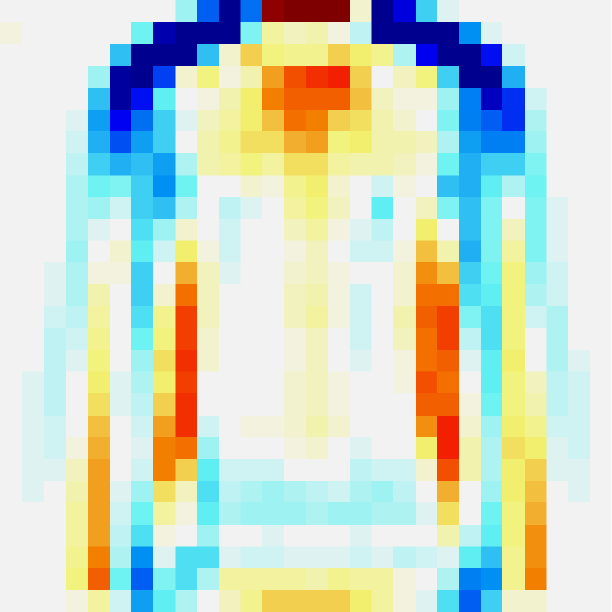}}
\subfigure[LR-NS]{\includegraphics[width=\parawidthExpFour]{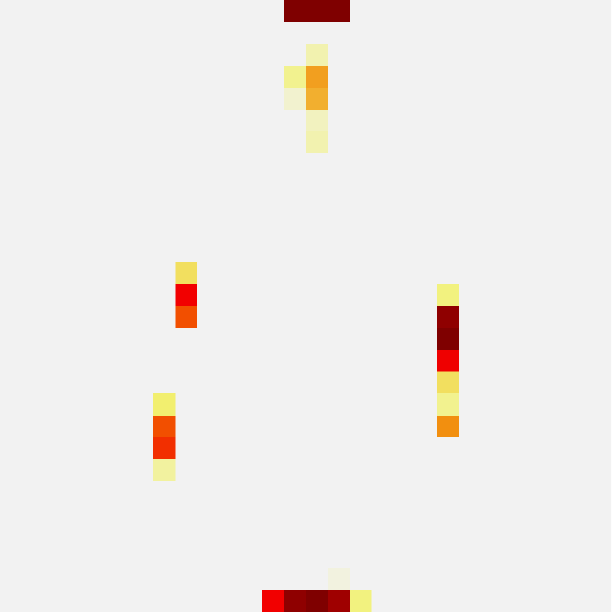}}
\subfigure[PLNN]{\includegraphics[width=\parawidthExpFour]{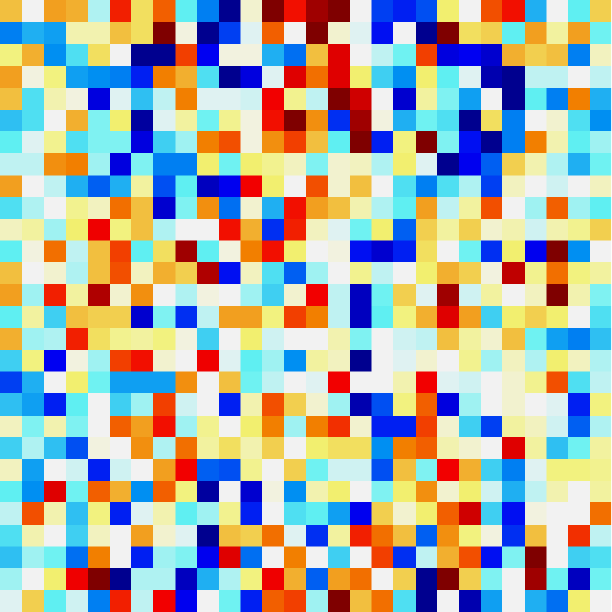}}
\subfigure[PLNN-NS]{\includegraphics[width=\parawidthExpFour]{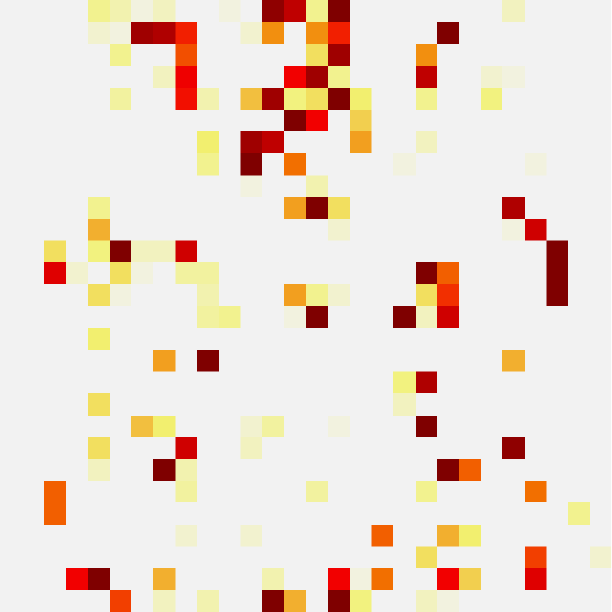}}
\includegraphics[height=\paraheightLegend]{Figures/legend}

\subfigure[Avg. Image]{\includegraphics[width=\parawidthExpFour]{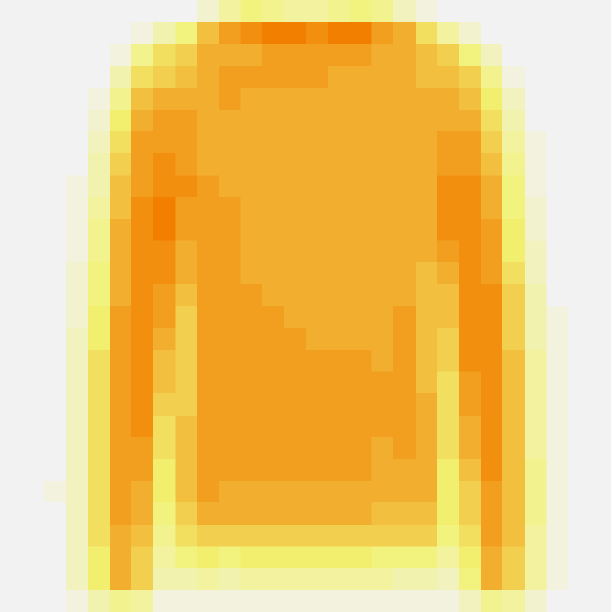}}
\subfigure[LR-F]{\includegraphics[width=\parawidthExpFour]{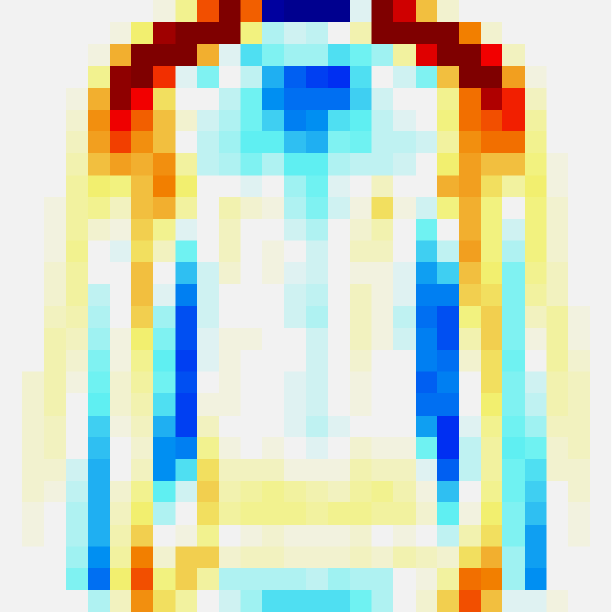}}
\subfigure[LR-NSF]{\includegraphics[width=\parawidthExpFour]{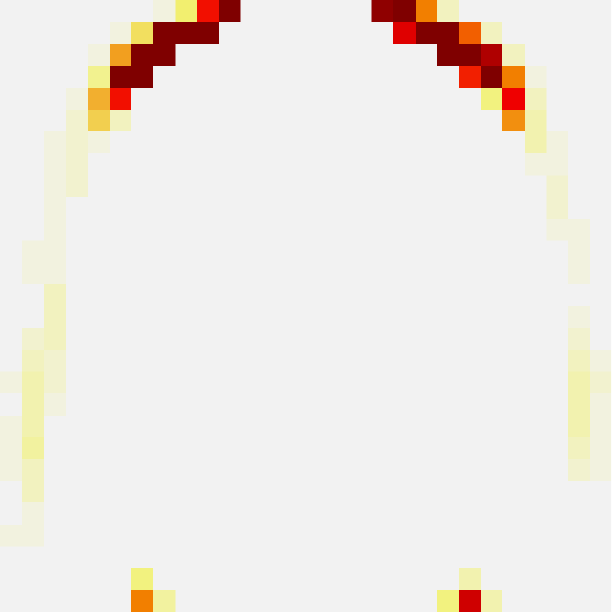}}
\subfigure[PLNN]{\includegraphics[width=\parawidthExpFour]{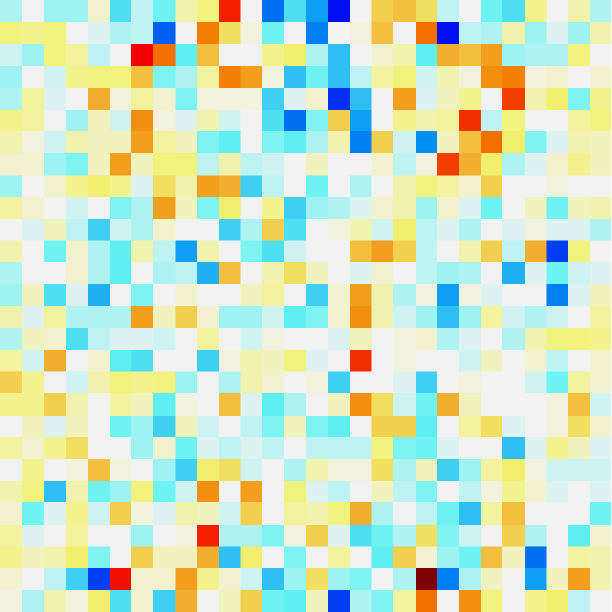}}
\subfigure[PLNN-NS]{\includegraphics[width=\parawidthExpFour]{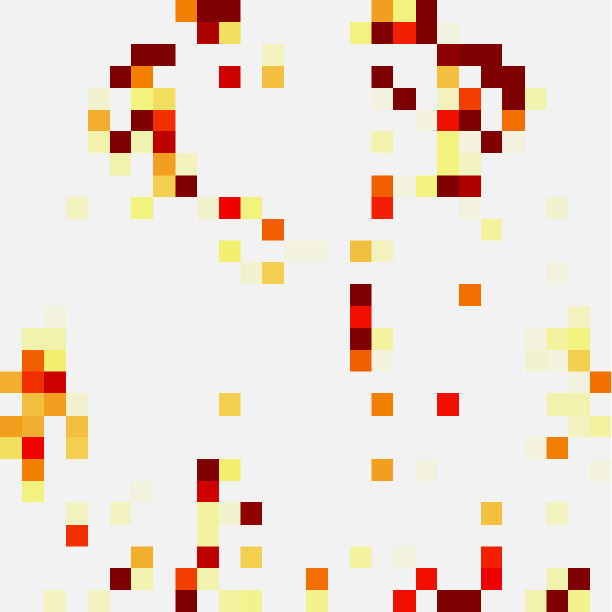}}
\includegraphics[height=\paraheightLegend]{Figures/legend}
\end{minipage}
\caption{The decision features of all models on FMNIST-2.
(a)-(e) and (f)-(j) show the average image and the decision features of all models for \emph{Coat} and \emph{Pullover}, respectively.
For PLNN and PLNN-NS, we show the decision features of the LLC whose convex polytope contains the most instances.
}
\label{Fig:decision_boundary_24}
\end{figure}

In Figures~\ref{Fig:decision_boundary_89} and~\ref{Fig:decision_boundary_24}, the decision features of PLNN-NS highlight similar image parts as LR-NS and LR-NSF, and are much easier to understand than the decision features of PLNN.
In particular, in Figure~\ref{Fig:decision_boundary_24}, the decision features of PLNN-NS clearly highlight the collar and breast of \emph{Coat}, and the shoulder of \emph{Pullover}, which are much easier to understand than the cluttered features of PLNN.
These results demonstrate the effectiveness of non-negative and sparse constraints in selecting meaningful features.
Moreover, the decision features of PLNN-NS capture more details than LR-NS and LR-NSF, thus PLNN-NS achieves a comparable accuracy with PLNN, and significantly outperforms LR-NS and LR-NSF on FMNIST-2.

In summary, the decision features of LLCs are easy to understand, and the non-negative and sparse constraints are highly effective in improving the interpretability of the decision features of LLCs.

\newcommand{\parawidthExpFive}{15mm}
\newcommand{\paraheightLegendII}{15mm}
\begin{figure}[t]
\centering
\subfigure[$\mathbf{z}_6^{(2)}$]{\includegraphics[width=\parawidthExpFive]{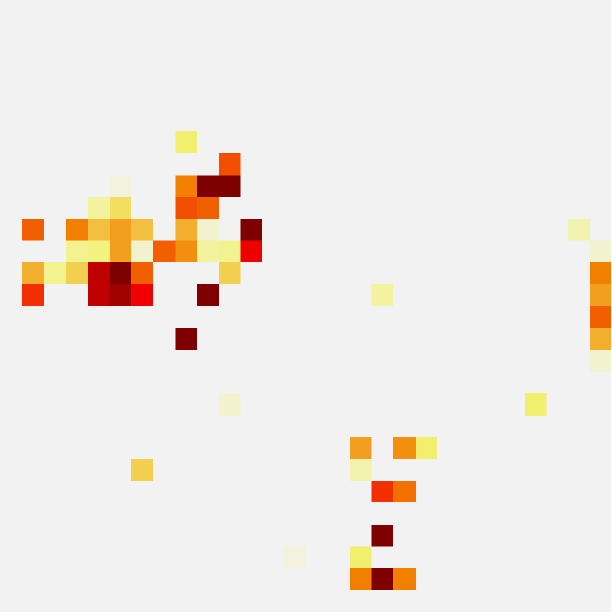}}
\subfigure[$\mathbf{z}_{11}^{(2)}$]{\includegraphics[width=\parawidthExpFive]{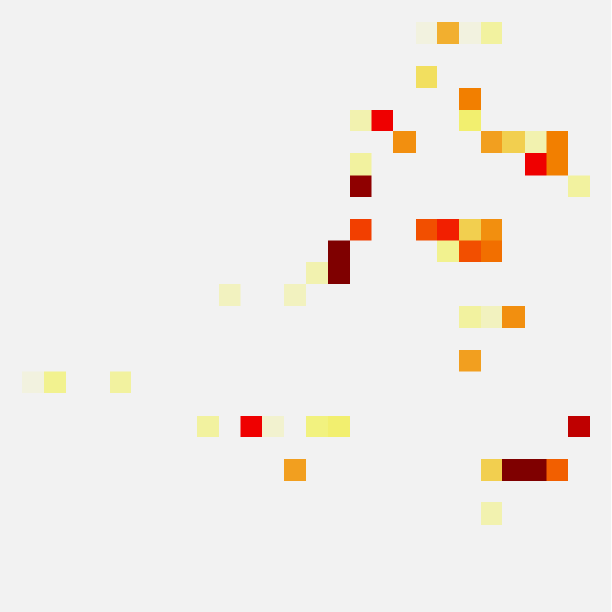}}
\subfigure[$\mathbf{z}_2^{(3)}$]{\includegraphics[width=\parawidthExpFive]{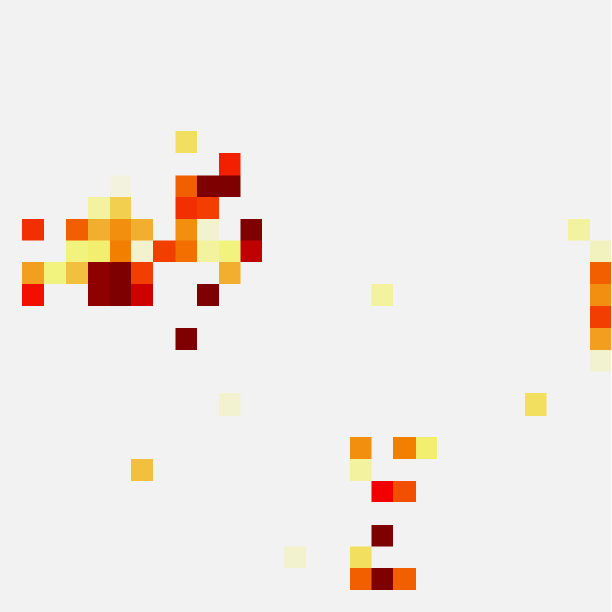}}
\subfigure[$\mathbf{z}_4^{(3)}$]{\includegraphics[width=\parawidthExpFive]{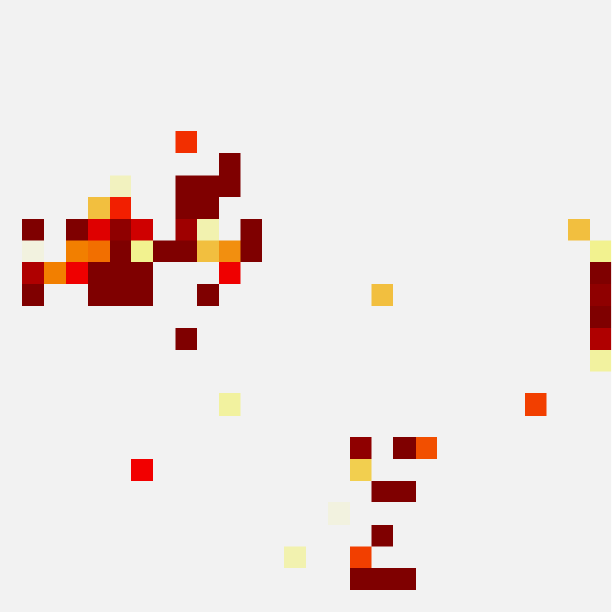}}
\includegraphics[height=\paraheightLegendII]{Figures/legend}

\subfigure[$\mathbf{z}_4^{(2)}$]{\includegraphics[width=\parawidthExpFive]{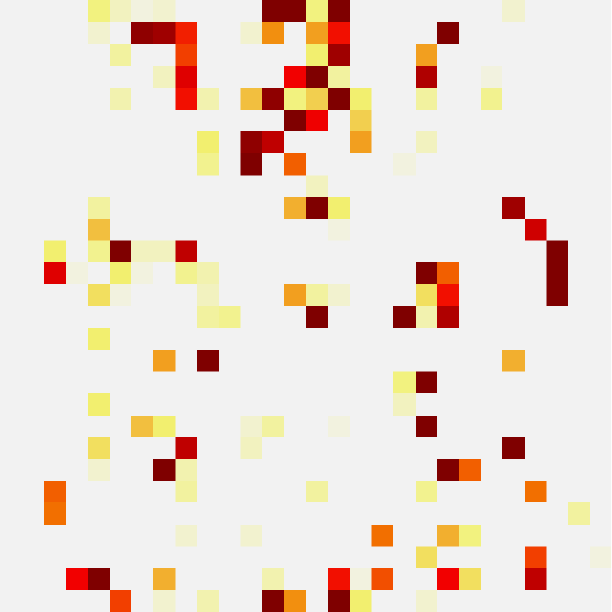}}
\subfigure[$\mathbf{z}_5^{(2)}$]{\includegraphics[width=\parawidthExpFive]{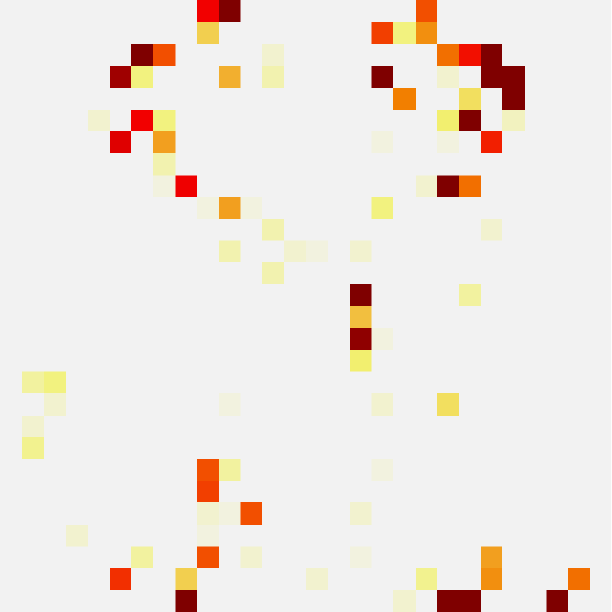}}
\subfigure[$\mathbf{z}_8^{(2)}$]{\includegraphics[width=\parawidthExpFive]{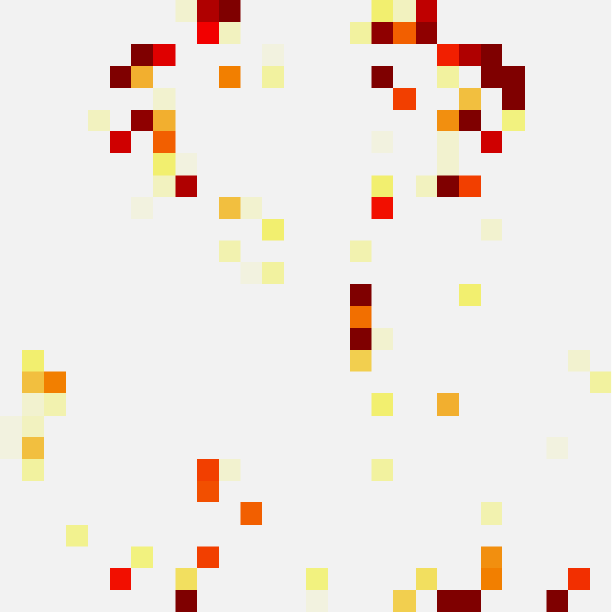}}
\subfigure[$\mathbf{z}_2^{(3)}$]{\includegraphics[width=\parawidthExpFive]{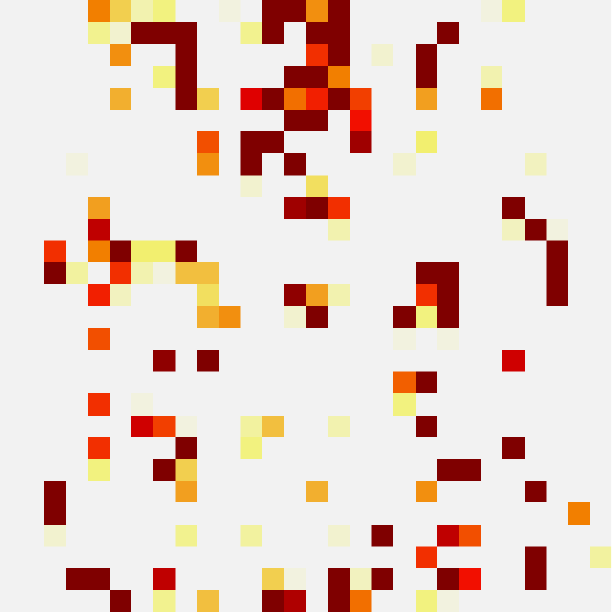}}
\includegraphics[height=\paraheightLegendII]{Figures/legend}
\caption{(a)-(d) show the PBFs of the PLNN-NS on FMNIST-1. (e)-(h) show the PBFs of the PLNN-NS on FMNIST-2.}
\label{Fig:polytope_bound}
\end{figure}

\begin{table}[t]
\centering
\caption{The PBs of the top-3 convex polytopes (CP) containing the most instances in FMNIST-1. ``/'' indicates a redundant linear inequality. Accuracy is the training accuracy of LLC on each CP.}
\label{Table:exp5_conf1}
\vspace{-2mm}\small
\begin{tabular}{|c|c|c|c|c|c|c|c|}
\hline
CP	& 	$\mathbf{z}_6^{(2)}$		&	$\mathbf{z}_{11}^{(2)}$		&	$\mathbf{z}_2^{(3)}$		&	$\mathbf{z}_4^{(3)}$  	&	\#\emph{Ankle Boot}		&	\#\emph{Bag}	&	Accuracy 	\\ \hline
	1		&	/					&	$>0$						&	$>0$					&	/					&	3,991				&	3,997	&	0.999	\\ \hline
	2		&	$\leq 0$				&	$>0$						&	/					&	$\leq 0$				&	9					&	0		&	1.000	\\ \hline
	3		&	/					&	$\leq 0$					&	/					&	$>0$					& 	0					&	3		&	1.000	\\ \hline
\end{tabular}
\end{table}

\begin{table}[t]
\centering
\caption{The PBs of the top-3 convex polytopes (CP) containing the most instances in FMNIST-2. Accuracy is the training accuracy of LLC on each CP.}
\label{Table:exp5_conf2}
\vspace{-2mm}\small
\begin{tabular}{|c|c|c|c|c|c|c|c|}
\hline
CP 	&	$\mathbf{z}_4^{(2)}$	&	$\mathbf{z}_5^{(2)}$  &	$\mathbf{z}_8^{(2)}$	&	$\mathbf{z}_2^{(3)}$  &	\#\emph{Coat}	&\#\emph{Pullover}	&	Accuracy 	\\ \hline
	
	1		&	$>0$		&	$>0$		&	$>0$		&	$>0$ 	&	3,932		&	3,942	&	0.894	\\ \hline
	
	2		&	$>0$		&	$\leq 0$	&	$>0$		&	$>0$		&	32			&	10		&	0.905	\\ \hline
	
	3		&	$>0$		&	$\leq 0$	& 	$\leq 0$	&	$>0$		& 	18			&	0		&	0.944	\\ \hline
\end{tabular}
\end{table}

\subsection{Are PBFs of LLCs Easy to Understand?}
\label{sec:apleu}
\nop{The polytope boundaries defines the domain of each LLC.}
The \textbf{polytope boundary features} (\textbf{PBFs}) of \textbf{polytope boundaries} (\textbf{PBs}) interpret why an instance is contained in the convex polytope of a LLC.
In this subsection, we study the semantical meaning of PBFs.
Limited by space, we only use the PLNN-NS models trained on FMNIST-1 and FMINST-2 as the target model to interpret.
The LLCs of PLNN-NS are computed by $OpenBox$.

\nop{
We study the polytope boundaries of the LLCs, and found that these polytope boundaries also have strong semantical meanings, which are useful in explaining the classification result of an instance.
}

Recall that a PB is defined by a linear inequality $\mathbf{z}_i^{(l)} \in \psi(\mathbf{c}_i^{(l)})$, where the PBFs are the coefficients of $\mathbf{x}$ in $\mathbf{z}_i^{(l)}$.
Since the activation function is ReLU, $\mathbf{z}_i^{(l)} \in \psi(\mathbf{c}_i^{(l)})$ is either $\mathbf{z}_i^{(l)} > 0$ or $\mathbf{z}_i^{(l)} \leq 0$.
Since the values of PBFs are non-negative for PLNN-NS, for a convex polytope $P_h$, if $\mathbf{z}_i^{(l)} > 0$, then the images in $P_h$ strongly correlate with the PBFs of $\mathbf{z}_i^{(l)}$; 
if $\mathbf{z}_i^{(l)} \leq 0$, then the images in $P_h$ are not strongly correlated with the PBFs of $\mathbf{z}_i^{(l)}$.

\nop{
The PBFs are the set of decision features of a PB.
According to the proof of Theorem~\ref{thm:polytope}, a PB of a convex polytope $P_h$ is defined by a linear inequality $\mathbf{z}_i^{(l)} \in \psi(\mathbf{c}_i^{(l)})$ in $Q_h$. 
By Equation~\ref{eqn:simpleform}, $\mathbf{z}_i^{(l)}$ is a linear function with respect to $\mathbf{x}$. 
The coefficient of $\mathbf{x}$ in $\mathbf{z}_i^{(l)}$ is the PBFs of the PB.
Since we adopt ReLU as the activation function, $\mathbf{z}_i^{(l)} \in \psi(\mathbf{c}_i^{(l)})$ is either $\mathbf{z}_i^{(l)} > 0$ or $\mathbf{z}_i^{(l)} \leq 0$.
For PLNN-NS, the PBFs are non-negative.
Therefore, if $\mathbf{z}_i^{(l)} > 0$ defines a PB of $P_h$, then $P_h$ contains the images that have the PBFs of $\mathbf{z}_i^{(l)}$; 
if $\mathbf{z}_i^{(l)} \leq 0$, then $P_h$ does not contain the images that have the PBFs.
}

\nop{
Here, $\mathbf{z}_i^{(l)}$ is a linear function with respect to $\mathbf{x}$, and the coefficient of $\mathbf{x}$ in $\mathbf{z}_i^{(l)}$ is the vector of importance weights of the PBFs.
}

\nop{
Figure~\ref{Fig:polytope_bound} shows the PBFs of the PLNN-NS models.
Take Figure~\ref{Fig:polytope_bound}(a) as an example, 
if $\mathbf{z}_6^{(2)} > 0$ is the PB of a convex polytope $P_h$, then $P_h$ contains the images that have the PBFs in Figure~\ref{Fig:polytope_bound}(a); 
if $\mathbf{z}_6^{(2)} \leq 0$, then $P_h$ does not contain the images that have the PBFs.
}

The above analysis of PBs and PBFs is demonstrated by the results in Tables~\ref{Table:exp5_conf1} and~\ref{Table:exp5_conf2}, and Figure~\ref{Fig:polytope_bound}.
Take the first convex polytope in Table~\ref{Table:exp5_conf1} as an example, the PBs are $\mathbf{z}_{11}^{(2)} > 0$ and $\mathbf{z}_{2}^{(3)} > 0$, whose PBFs in Figures~\ref{Fig:polytope_bound}(b)-(c) show the features of \emph{Ankle Boot} and \emph{Bag}, respectively.
Therefore, the convex polytope contains images of both \emph{Ankle Boot} and \emph{Bag}.
A careful study of the other results suggests that the PBFs of the convex polytopes are easy to understand and accurately describe the images in each convex polytope.

\nop{
This provide comprehensive insights into the behavior of a PLNN-NS.

In sum, the PBFs of the convex polytopes computed by $OpenBox$ are easy to understand and provide comprehensive insights into the behavior of a PLNN-NS.
}

\nop{
A careful study of the PBs of the other convex polytopes in Table~\ref{Table:exp5_conf1} and Table~\ref{Table:exp5_conf2} finds out that the PBFs in Figure~\ref{Fig:polytope_bound} accurately identifies the features of the images in each convex polytope.
}

\nop{
Similar results on FMNIST-2 can be observed from the PBFs in Figures~\ref{Fig:polytope_bound}(e)-(h) and the results in Table~\ref{Table:exp5_conf2}.
}

\nop{
As shown in Figures~\ref{Fig:polytope_bound}(b)-(c), the PBFs of $\mathbf{z}_{11}^{(2)} > 0$ and $\mathbf{z}_{2}^{(3)} > 0$ correspond to the features of \emph{Ankle Boot} and \emph{Bag}, respectively. 
}

\nop{
Therefore, the first convex polytope in Table~\ref{Table:exp5_conf1} contains images of both \emph{Ankle Boot} and \emph{Bag}.
Similar results on FMNIST-2 can be observed from the PBFs in Figures~\ref{Fig:polytope_bound}(e)-(h) and the results in Table~\ref{Table:exp5_conf2}.
}

We can also see that the PBFs in Figure~\ref{Fig:polytope_bound} look similar to the decision features of PLNN-NS in Figures~\ref{Fig:decision_boundary_89} and~\ref{Fig:decision_boundary_24}.
This shows the strong correlation between the features learned by different neurons of PLNN-NS, which is probably caused by the hierarchy network structure.
Due to the strong correlation between neurons, the number of configurations in $\mathcal{C}$ is much less than $k^N$, as shown in Table~\ref{Table:netstruct}.

Surprisingly, as shown in Table~\ref{Table:exp5_conf2}, the top-1 convex polytope on FMNIST-2 contains more than 98\% of the training instances. 
On these instances, the training accuracy of LLC is much higher than the training accuracies of LR-NS and LR-NSF.
This means that the training instances in the top-1 convex polytope are much easier to be linearly separated than all training instances in FMNIST-2.
From this perspective, the behavior of PLNN-NS is like a ``divide and conquer'' strategy, which set aside a small proportion of instances that hinder the classification accuracy such that the majority of the instances can be better separated by a LLC.
As shown by the top-2 and top-3 convex polytopes in Table~\ref{Table:exp5_conf2}, the set aside instances are grouped in their own convex polytopes, where the corresponding LLCs also achieve a very high accuracy.
Table~\ref{Table:exp5_conf1} shows similar phenomenon on FMNIST-1. However, since the instances in FMNIST-1 are easy to be linearly separated, the training accuracy of PLNN-NS marginally outperforms LR-NS and LR-NSF.

\subsection{Can We Hack a Model Using OpenBox?}
\nop{\todo{Explain what you mean by ``hack''.}}

Knowing what an intelligent machine ``thinks'' provides us the privilege to ``hack'' it. 
Here, to hack a target model is to significantly change its prediction on an instance $\mathbf{x}\in\mathcal{X}$ by modifying as few features of $\mathbf{x}$ as possible.
In general, the biggest change of prediction is achieved by modifying the most important decision features.
A more precise interpretation on the target model reveals the important decision features more accurately, thus 
requires to modify less features to achieve a bigger change of prediction.
Following this idea, we apply LIME and $OpenBox$ to hack PLNN-NS, and compare the quality of their interpretations by comparing the change of PLNN-NS's prediction when modifying the same number of decision features.

\nop{
Precisely finding the decision features that dominates the classification of $\mathbf{x}$ is the key to achieve a high hacking performance.

To achieve a good hacking performance, we have to find

In general, modifying a small number of important decision features makes a bigger change of prediction than modifying a large number of unimportant features that are less relevant to the classification of $\mathbf{x}$.
This requires to find the important decision features as precise as possible.

In general, a more precise interpretation on the target model finds out more important decision features, thus requires to modify less features to achieve a significant change of prediction.
Following this idea, we apply LIME and $OpenBox$ to hack PLNN-NS, and compare the quality of their interpretations by comparing their hacking performance in a quantitative manner.
}

For an instance $\mathbf{x}\in\mathcal{X}$, denote by $\gamma\in\mathbb{R}^d$ the decision features for the classification of $\mathbf{x}$. We hack PLNN-NS by setting the values of a few top-weighted decision features in $\mathbf{x}$ to zero, such that the prediction of PLNN-NS on $\mathbf{x}$ changes significantly. 
The change of prediction is evaluated by two measures as follows.
First, the \textbf{change of prediction probability} (\textbf{CPP}) is the absolute change of the probability of classifying $\mathbf{x}$ as a positive instance.
Second, the \textbf{number of label-changed instance} (\textbf{NLCI}) is the number of instances whose predicted label changes after being hacked.
Again, due to the inefficiency of LIME, we use the sampled data sets in Section~\ref{sec:exact_consist} for evaluation.

In Figure~\ref{Fig:hack}, the average CPP and NLCI of $OpenBox$ are always higher than LIME on both data sets. 
This demonstrates that the interpretations computed by $OpenBox$ are more effective than LIME when they are applied to hack the target model.

Interestingly, the advantage of $OpenBox$ is more significant on FMNIST-1 than on FMNIST-2. This is because, as shown in Figure~\ref{Fig:exp2}(a), the prediction probabilities of most instances in FMNIST-1 are either 1.0 or 0.0, which provides little gradient information for LIME to accurately approximate the classification function of the PLNN-NS. 
In this case, the decision features computed by LIME cannot describe the exact behavior of the target model.

In summary, 
since $OpenBox$ produces the exact and consistent interpretations for a target model, it achieves an advanced hacking performance over LIME.

\nop{
as demonstrated by the advanced hacking performance $OpenBox$, has a great potential in discovering the weakness and testing the robustness of a PLNN model before its deployment.
}

\nop{
$OpenBox$ is highly effective in hacking the PLNN-NS, thus it has a great potential in discovering the weakness and testing the robustness of a PLNN model before deployment.
}

\newcommand{\parawidthExpHack}{42mm}
\begin{figure}[t]
\centering
\subfigure[FMNIST-1]{\includegraphics[width=\parawidthExpHack]{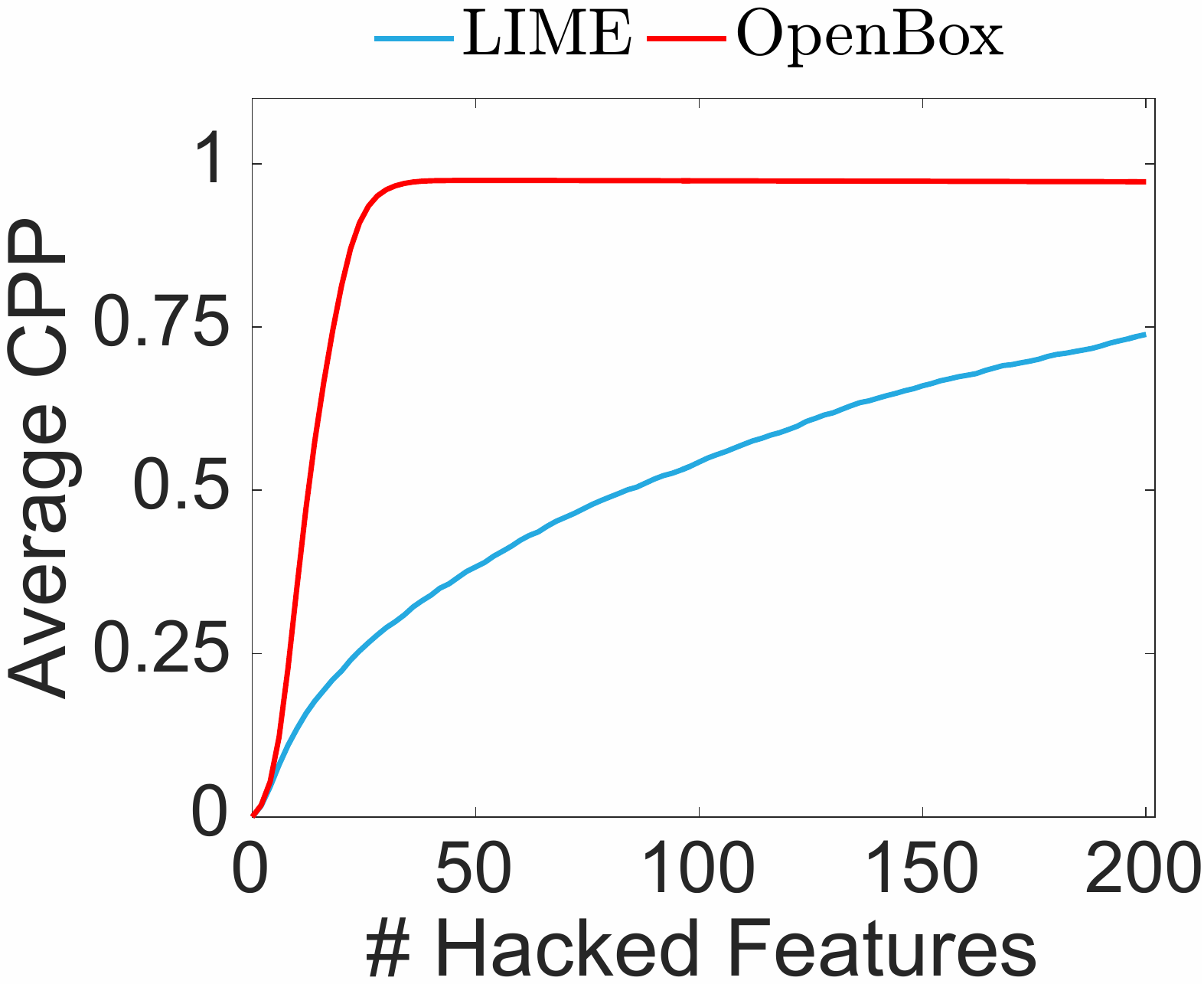}}
\subfigure[FMNIST-2]{\includegraphics[width=\parawidthExpHack]{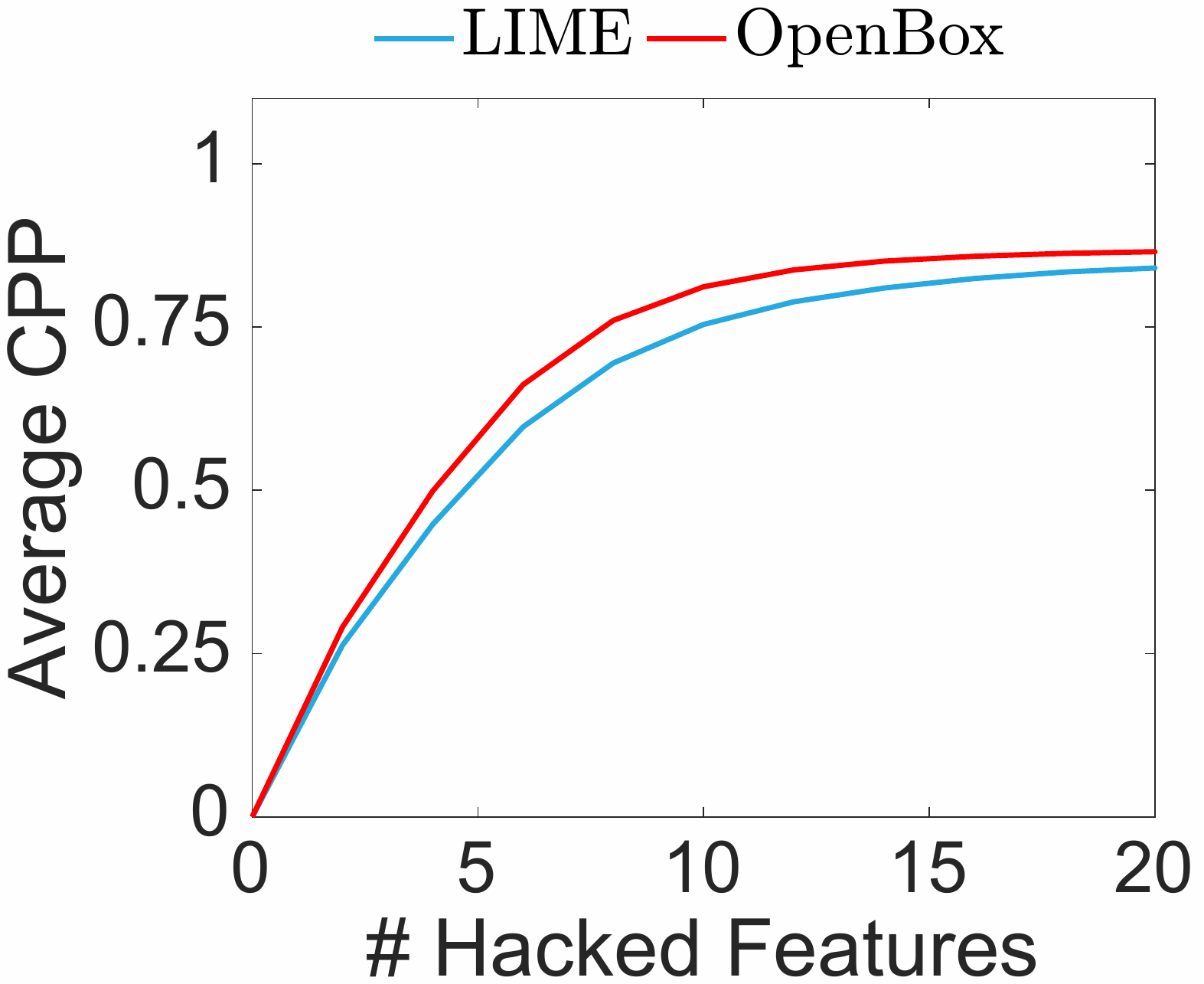}}
\subfigure[FMNIST-1]{\includegraphics[width=\parawidthExpHack]{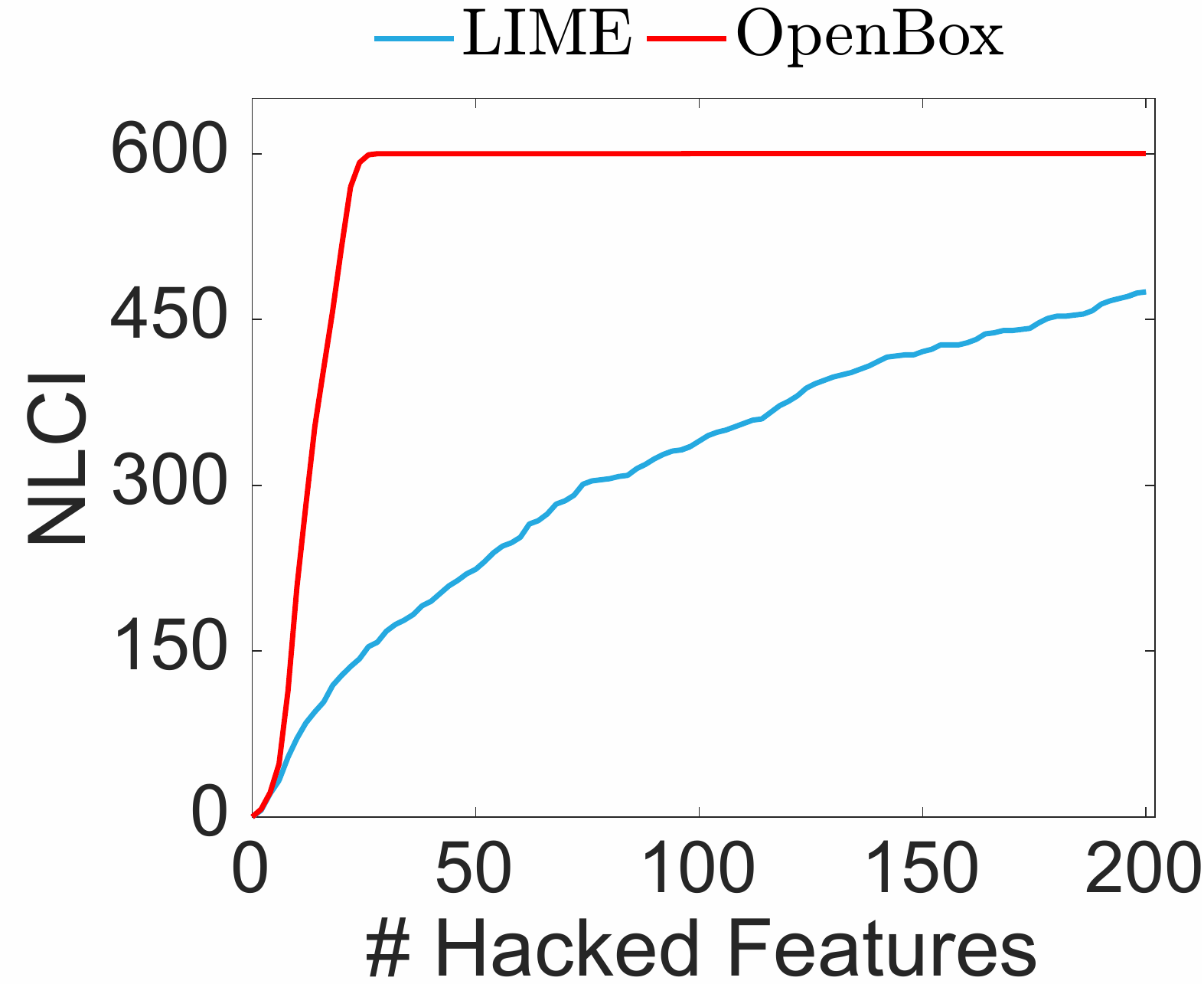}}
\subfigure[FMNIST-2]{\includegraphics[width=\parawidthExpHack]{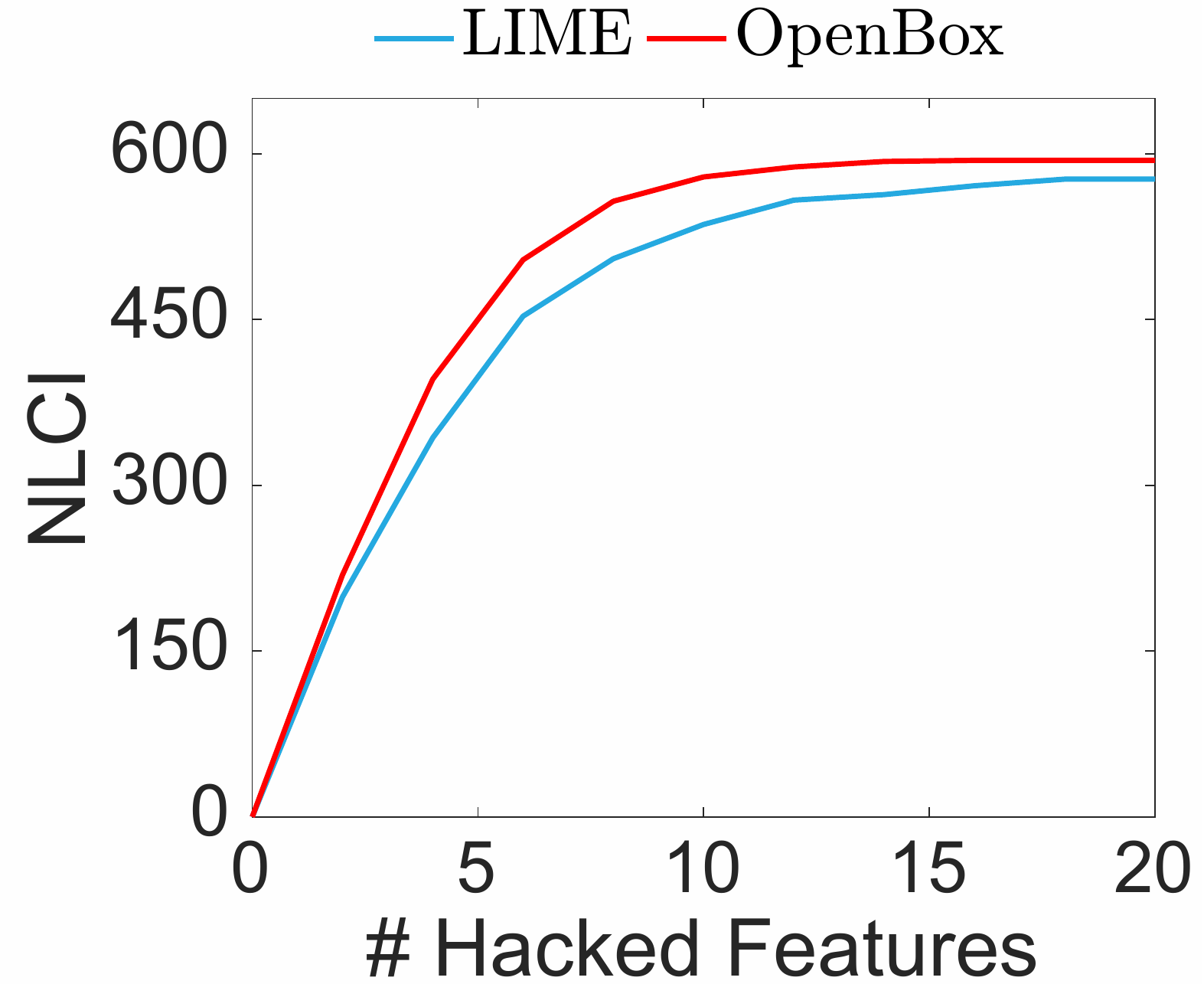}}
\caption{The hacking performance of LIME and $OpenBox$. (a)-(b) show the Average CPP. (c)-(d) show the NLCI.}
\label{Fig:hack}
\end{figure}

\subsection{Can We Debug a Model Using OpenBox?}
Intelligent machines are not perfect and predictions fail occasionally. 
When such failure occurs, we can apply $OpenBox$ to interpret why an instance is mis-classified.

Figure~\ref{Fig:debug_fmnist} shows some images that are mis-classified by PLNN-NS with a high probability.
In Figures~\ref{Fig:debug_fmnist}(a)-(c), the original image is a \emph{Coat}, however, since the scattered mosaic pattern on the cloth hits more features of \emph{Pullover} than \emph{Coat}, the original image is classified as a \emph{Pullover} with a high probability.
In Figures~\ref{Fig:debug_fmnist}(d)-(f), the original image is a \emph{Pullover}, however, it is mis-classified as a \emph{Coat} because the white collar and breast hit the typical features of \emph{Coat}, and the dark shoulder and sleeves miss the most significant features of \emph{Pullover}.
Similarly, the \emph{Ankle Boot} in Figure~\ref{Fig:debug_fmnist}(g) highlights more features on the upper left corner, thus it is mis-classified as a \emph{Bag}.
The \emph{Bag} in Figure~\ref{Fig:debug_fmnist}(j) is mis-classified as an \emph{Ankle Boot} because it hits the features of ankle and heel of \emph{Ankle Boot}, however, misses the typical features of \emph{Bag} on the upper left corner.

In conclusion, as demonstrated by Figure~\ref{Fig:debug_fmnist}, $OpenBox$ accurately interprets the mis-classifications, which is potentially useful in debugging abnormal behaviors of the interpreted model.


\newcommand{\parawidthExpDebugI}{13mm}
\begin{figure}[t]
\centering

\subfigure[CO]{\includegraphics[width=\parawidthExpDebugI]{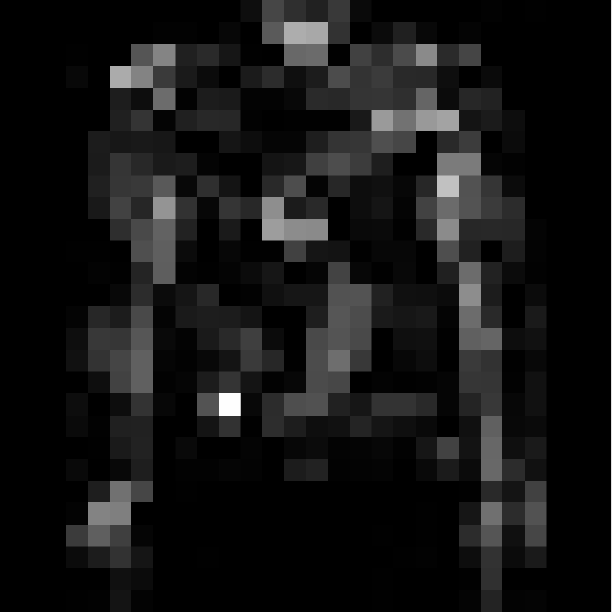}}
\subfigure[CO: 0.04]{\includegraphics[width=\parawidthExpDebugI]{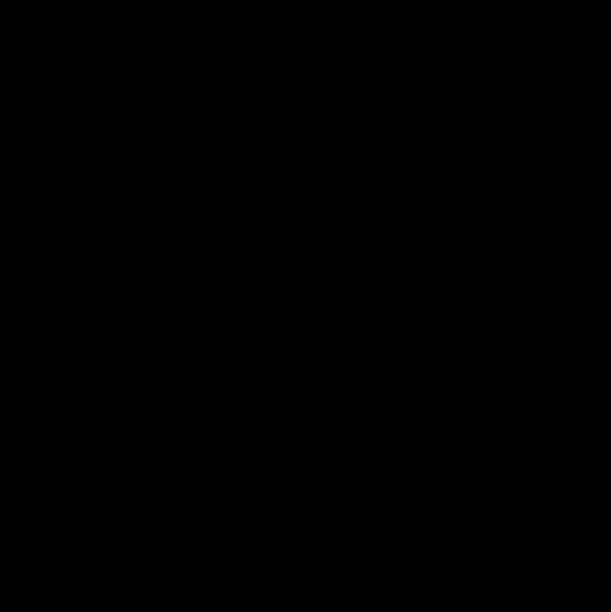}}
\subfigure[PU: 0.96]{\includegraphics[width=\parawidthExpDebugI]{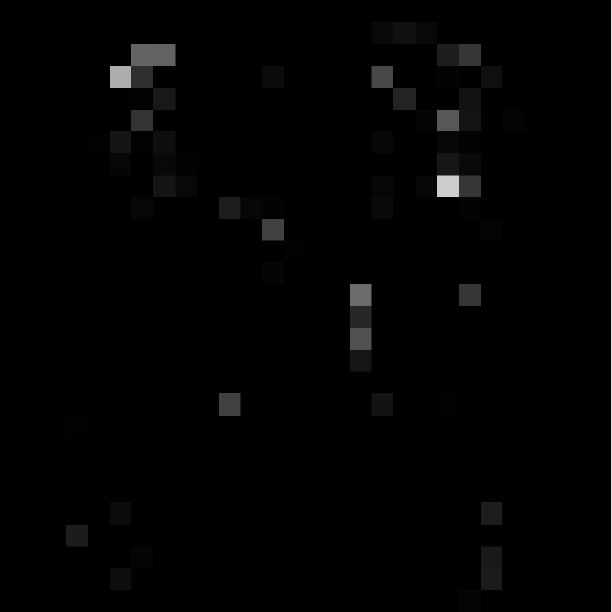}}
\hspace{2mm}
\subfigure[PU]{\includegraphics[width=\parawidthExpDebugI]{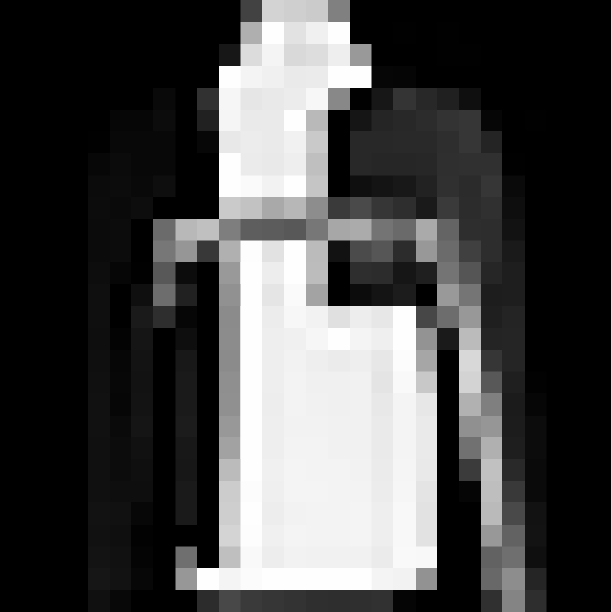}}
\subfigure[CO: 1.00]{\includegraphics[width=\parawidthExpDebugI]{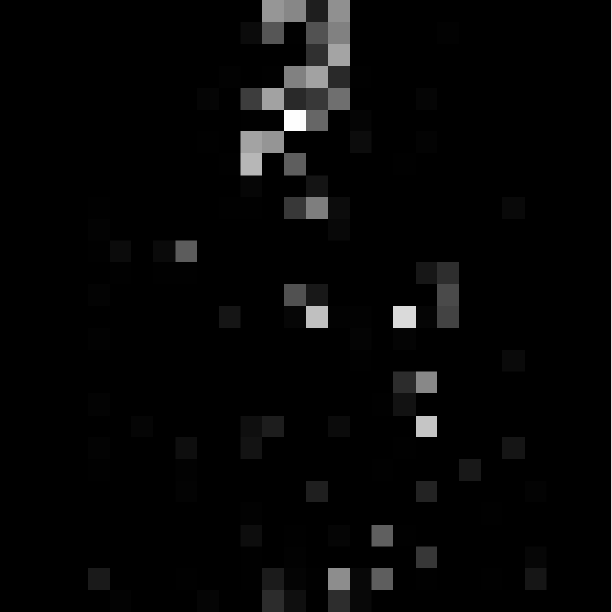}}
\subfigure[PU: 0.00]{\includegraphics[width=\parawidthExpDebugI]{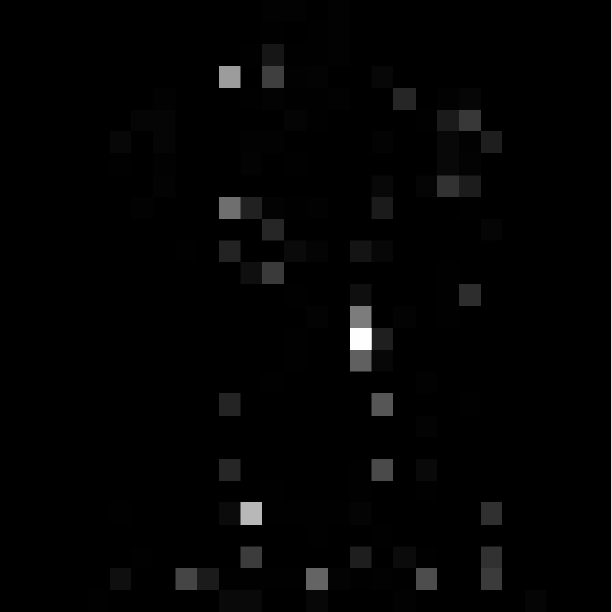}}
\subfigure[AB]{\includegraphics[width=\parawidthExpDebugI]{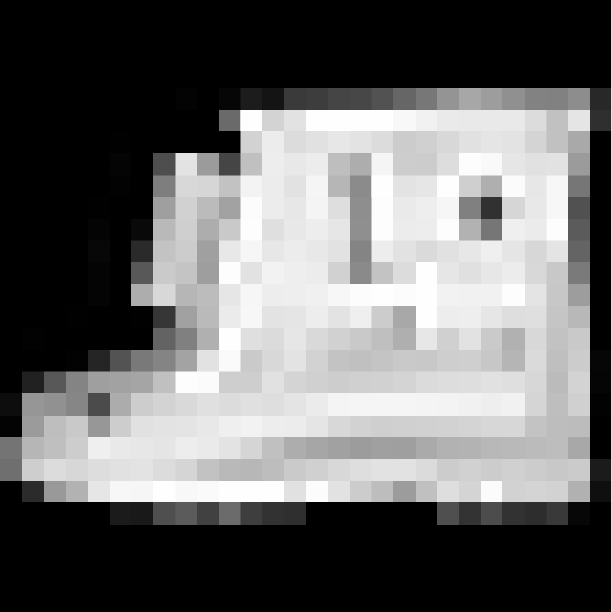}}
\subfigure[AB: 0.16]{\includegraphics[width=\parawidthExpDebugI]{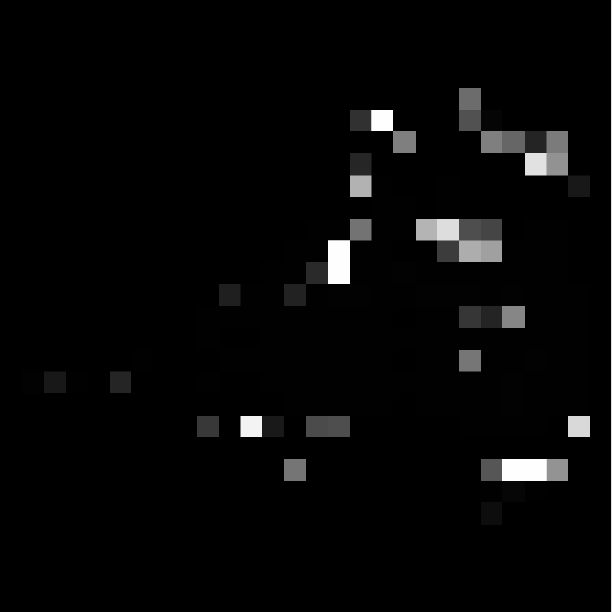}}
\subfigure[BG: 0.84]{\includegraphics[width=\parawidthExpDebugI]{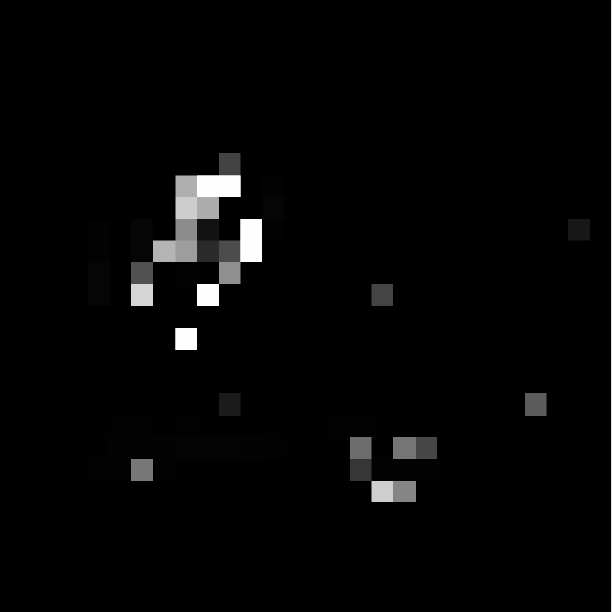}}
\hspace{2mm}
\subfigure[BG]{\includegraphics[width=\parawidthExpDebugI]{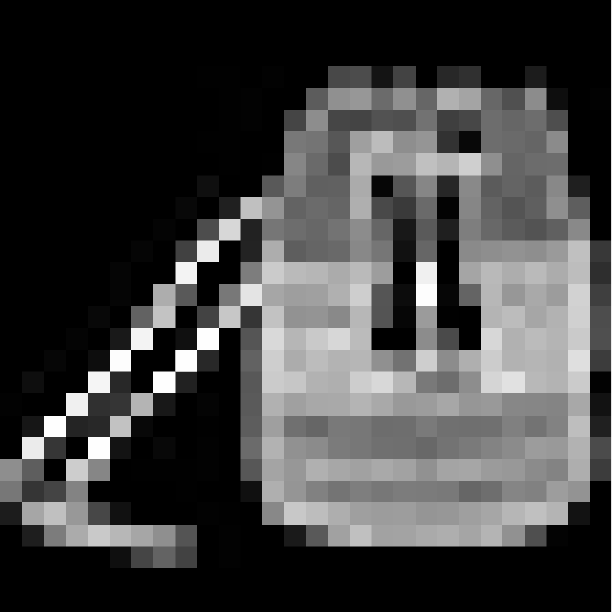}}
\subfigure[AB: 1.00]{\includegraphics[width=\parawidthExpDebugI]{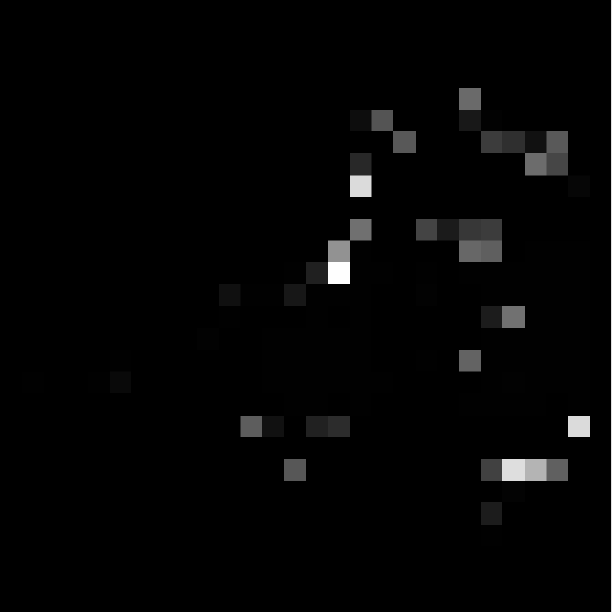}}
\subfigure[BG: 0.00]{\includegraphics[width=\parawidthExpDebugI]{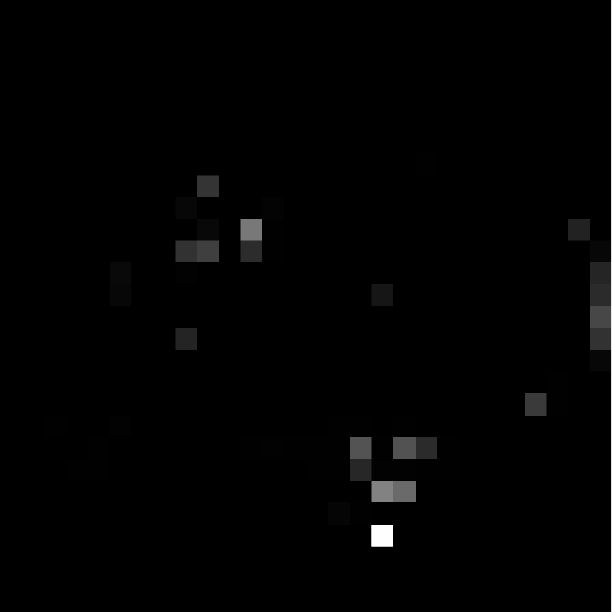}}
\caption{The mis-classified images of (a) \emph{Coat} (\textbf{CO}), (d) \emph{Pullover} (\textbf{PU}), (g) \emph{Ankle Boot} (\textbf{AB}), and (j) \emph{Bag} (\textbf{BG}).
(a), (d), (g) and (j) show the original images.
For the rest subfigures, the caption shows the prediction probability of the corresponding class; the image shows the decision features supporting the prediction of the corresponding class.
}
\label{Fig:debug_fmnist}
\end{figure}

\section{Conclusions and Future Work}
\label{sec:con}
In this paper, we tackle the challenging problem of interpreting PLNNs.
By studying the states of hidden neurons and the configuration of a PLNN, we prove that a PLNN is mathematically equivalent to a set of LLCs, which can be efficiently computed by the proposed $OpenBox$ method.
Extensive experiments show that the decision features and the polytope boundary features of LLCs provide exact and consistent interpretations on the overall behavior of a PLNN.
Such interpretations are highly effective in hacking and debugging PLNN models.
As future work, we will extend our work to interpret more general neural networks that adopt smooth activation functions, such as sigmoid and $\tanh$.

\nop{
propose $OpenBox$ to compute a set of local linear classifiers that is mathematically equivalent to the target PLNN to be interpreted.

First, we prove that a PLNN is mathematically equivalent to a set of local linear classifiers, each of which is a linear classifier that classifies a group of instances within a convex polytope in the input space.

Second, We propose a method named $OpenBox$ to provide an exact interpretation of a PLNN by computing its equivalent set of local linear classifiers in closed form.

Third, we explain the classification result of each instance by the decision features of its local linear classifier. Since all instances in the same convex polytope share exactly the same local linear classifier, our explanations are consistent per convex polytope.

Fourth, we also apply $OpenBox$ to study the effect of non-negative and sparse constraints on the interpretability of PLNNs. We find that a PLNN trained with these constraints selects meaningful features that dramatically improve the interpretability.

Last, we conduct extensive experiments on both synthetic and real-world data sets to demonstrate the effectiveness of our method.
}

\bibliographystyle{ACM-Reference-Format}
\bibliography{Reference}


\begin{thebibliography}{48}


\ifx \showCODEN    \undefined \def \showCODEN     #1{\unskip}     \fi
\ifx \showDOI      \undefined \def \showDOI       #1{#1}\fi
\ifx \showISBNx    \undefined \def \showISBNx     #1{\unskip}     \fi
\ifx \showISBNxiii \undefined \def \showISBNxiii  #1{\unskip}     \fi
\ifx \showISSN     \undefined \def \showISSN      #1{\unskip}     \fi
\ifx \showLCCN     \undefined \def \showLCCN      #1{\unskip}     \fi
\ifx \shownote     \undefined \def \shownote      #1{#1}          \fi
\ifx \showarticletitle \undefined \def \showarticletitle #1{#1}   \fi
\ifx \showURL      \undefined \def \showURL       {\relax}        \fi
\providecommand\bibfield[2]{#2}
\providecommand\bibinfo[2]{#2}
\providecommand\natexlab[1]{#1}
\providecommand\showeprint[2][]{arXiv:#2}

\bibitem[\protect\citeauthoryear{Agrawal, Batra, and Parikh}{Agrawal
  et~al\mbox{.}}{2016}]%
        {agrawal2016analyzing}
\bibfield{author}{\bibinfo{person}{Aishwarya Agrawal}, \bibinfo{person}{Dhruv
  Batra}, {and} \bibinfo{person}{Devi Parikh}.}
  \bibinfo{year}{2016}\natexlab{}.
\newblock \showarticletitle{Analyzing the behavior of visual question answering
  models}.
\newblock \bibinfo{journal}{\emph{arXiv:1606.07356}} (\bibinfo{year}{2016}).
\newblock


\bibitem[\protect\citeauthoryear{Ba and Caruana}{Ba and Caruana}{2014}]%
        {ba2014deep}
\bibfield{author}{\bibinfo{person}{Jimmy Ba} {and} \bibinfo{person}{Rich
  Caruana}.} \bibinfo{year}{2014}\natexlab{}.
\newblock \showarticletitle{Do deep nets really need to be deep?}. In
  \bibinfo{booktitle}{\emph{NIPS}}. \bibinfo{pages}{2654--2662}.
\newblock


\bibitem[\protect\citeauthoryear{Bastani, Kim, and Bastani}{Bastani
  et~al\mbox{.}}{2017}]%
        {bastani2017interpreting}
\bibfield{author}{\bibinfo{person}{Osbert Bastani}, \bibinfo{person}{Carolyn
  Kim}, {and} \bibinfo{person}{Hamsa Bastani}.}
  \bibinfo{year}{2017}\natexlab{}.
\newblock \showarticletitle{Interpreting Blackbox Models via Model Extraction}.
\newblock \bibinfo{journal}{\emph{arXiv:1705.08504}} (\bibinfo{year}{2017}).
\newblock


\bibitem[\protect\citeauthoryear{Bishop}{Bishop}{2007}]%
        {bishop2007pattern}
\bibfield{author}{\bibinfo{person}{C Bishop}.} \bibinfo{year}{2007}\natexlab{}.
\newblock \showarticletitle{Pattern Recognition and Machine Learning
  (Information Science and Statistics)}.
\newblock \bibinfo{journal}{\emph{Springer, New York}} (\bibinfo{year}{2007}).
\newblock


\bibitem[\protect\citeauthoryear{Cao, Liu, Yang, Yu, Wang, Wang, Huang, Wang,
  Huang, et~al\mbox{.}}{Cao et~al\mbox{.}}{2015}]%
        {cao2015look}
\bibfield{author}{\bibinfo{person}{C. Cao}, \bibinfo{person}{X. Liu},
  \bibinfo{person}{Y Yang}, \bibinfo{person}{Y. Yu}, \bibinfo{person}{J. Wang},
  \bibinfo{person}{Z. Wang}, \bibinfo{person}{Y. Huang}, \bibinfo{person}{L.
  Wang}, \bibinfo{person}{C. Huang}, {et~al\mbox{.}}}
  \bibinfo{year}{2015}\natexlab{}.
\newblock \showarticletitle{Look and think twice: Capturing top-down visual
  attention with feedback convolutional neural networks}. In
  \bibinfo{booktitle}{\emph{ICCV}}. \bibinfo{pages}{2956--2964}.
\newblock


\bibitem[\protect\citeauthoryear{Caron, McDonald, and Ponic}{Caron
  et~al\mbox{.}}{1989}]%
        {caron1989degenerate}
\bibfield{author}{\bibinfo{person}{RJ Caron}, \bibinfo{person}{JF McDonald},
  {and} \bibinfo{person}{CM Ponic}.} \bibinfo{year}{1989}\natexlab{}.
\newblock \showarticletitle{A degenerate extreme point strategy for the
  classification of linear constraints as redundant or necessary}.
\newblock \bibinfo{journal}{\emph{JOTA}} \bibinfo{volume}{62},
  \bibinfo{number}{2} (\bibinfo{year}{1989}), \bibinfo{pages}{225--237}.
\newblock


\bibitem[\protect\citeauthoryear{Che, Purushotham, Khemani, and Liu}{Che
  et~al\mbox{.}}{2015}]%
        {che2015distilling}
\bibfield{author}{\bibinfo{person}{Z. Che}, \bibinfo{person}{S. Purushotham},
  \bibinfo{person}{R. Khemani}, {and} \bibinfo{person}{Y. Liu}.}
  \bibinfo{year}{2015}\natexlab{}.
\newblock \showarticletitle{Distilling knowledge from deep networks with
  applications to healthcare domain}.
\newblock \bibinfo{journal}{\emph{arXiv:1512.03542}} (\bibinfo{year}{2015}).
\newblock


\bibitem[\protect\citeauthoryear{Chorowski and Zurada}{Chorowski and
  Zurada}{2015}]%
        {chorowski2015learning}
\bibfield{author}{\bibinfo{person}{Jan Chorowski} {and}
  \bibinfo{person}{Jacek~M Zurada}.} \bibinfo{year}{2015}\natexlab{}.
\newblock \showarticletitle{Learning understandable neural networks with
  nonnegative weight constraints}.
\newblock \bibinfo{journal}{\emph{TNNLS}} \bibinfo{volume}{26},
  \bibinfo{number}{1} (\bibinfo{year}{2015}), \bibinfo{pages}{62--69}.
\newblock


\bibitem[\protect\citeauthoryear{Dosovitskiy and Brox}{Dosovitskiy and
  Brox}{2016}]%
        {dosovitskiy2016inverting}
\bibfield{author}{\bibinfo{person}{Alexey Dosovitskiy} {and}
  \bibinfo{person}{Thomas Brox}.} \bibinfo{year}{2016}\natexlab{}.
\newblock \showarticletitle{Inverting visual representations with convolutional
  networks}. In \bibinfo{booktitle}{\emph{CVPR}}. \bibinfo{pages}{4829--4837}.
\newblock


\bibitem[\protect\citeauthoryear{Erhan, Bengio, Courville, and Vincent}{Erhan
  et~al\mbox{.}}{2009}]%
        {erhan2009visualizing}
\bibfield{author}{\bibinfo{person}{D. Erhan}, \bibinfo{person}{Yoshua Bengio},
  \bibinfo{person}{A. Courville}, {and} \bibinfo{person}{P. Vincent}.}
  \bibinfo{year}{2009}\natexlab{}.
\newblock \showarticletitle{Visualizing higher-layer features of a deep
  network}.
\newblock \bibinfo{journal}{\emph{University of Montreal}}
  \bibinfo{volume}{1341} (\bibinfo{year}{2009}), \bibinfo{pages}{3}.
\newblock


\bibitem[\protect\citeauthoryear{Fong and Vedaldi}{Fong and Vedaldi}{2017}]%
        {fong2017interpretable}
\bibfield{author}{\bibinfo{person}{Ruth Fong} {and} \bibinfo{person}{Andrea
  Vedaldi}.} \bibinfo{year}{2017}\natexlab{}.
\newblock \showarticletitle{Interpretable Explanations of Black Boxes by
  Meaningful Perturbation}.
\newblock \bibinfo{journal}{\emph{arXiv:1704.03296}} (\bibinfo{year}{2017}).
\newblock


\bibitem[\protect\citeauthoryear{Frosst and Hinton}{Frosst and Hinton}{2017}]%
        {frosst2017distilling}
\bibfield{author}{\bibinfo{person}{Nicholas Frosst} {and}
  \bibinfo{person}{Geoffrey Hinton}.} \bibinfo{year}{2017}\natexlab{}.
\newblock \showarticletitle{Distilling a Neural Network Into a Soft Decision
  Tree}.
\newblock \bibinfo{journal}{\emph{arXiv:1711.09784}} (\bibinfo{year}{2017}).
\newblock


\bibitem[\protect\citeauthoryear{Ghorbani, Abid, and Zou}{Ghorbani
  et~al\mbox{.}}{2017}]%
        {ghorbani2017interpretation}
\bibfield{author}{\bibinfo{person}{Amirata Ghorbani}, \bibinfo{person}{Abubakar
  Abid}, {and} \bibinfo{person}{James Zou}.} \bibinfo{year}{2017}\natexlab{}.
\newblock \showarticletitle{Interpretation of Neural Networks is Fragile}.
\newblock \bibinfo{journal}{\emph{arXiv:1710.10547}} (\bibinfo{year}{2017}).
\newblock


\bibitem[\protect\citeauthoryear{Glorot, Bordes, and Bengio}{Glorot
  et~al\mbox{.}}{2011}]%
        {glorot2011deep}
\bibfield{author}{\bibinfo{person}{Xavier Glorot}, \bibinfo{person}{Antoine
  Bordes}, {and} \bibinfo{person}{Yoshua Bengio}.}
  \bibinfo{year}{2011}\natexlab{}.
\newblock \showarticletitle{Deep sparse rectifier neural networks}. In
  \bibinfo{booktitle}{\emph{ICAIS}}. \bibinfo{pages}{315--323}.
\newblock


\bibitem[\protect\citeauthoryear{Goodfellow, Bengio, and Courville}{Goodfellow
  et~al\mbox{.}}{2016}]%
        {Goodfellow-et-al-2016}
\bibfield{author}{\bibinfo{person}{Ian Goodfellow}, \bibinfo{person}{Yoshua
  Bengio}, {and} \bibinfo{person}{Aaron Courville}.}
  \bibinfo{year}{2016}\natexlab{}.
\newblock \bibinfo{booktitle}{\emph{Deep Learning}}.
\newblock \bibinfo{publisher}{MIT Press}.
\newblock
\newblock
\shownote{\url{http://www.deeplearningbook.org}.}


\bibitem[\protect\citeauthoryear{Goodfellow, Warde-Farley, Mirza, Courville,
  and Bengio}{Goodfellow et~al\mbox{.}}{2013}]%
        {goodfellow2013maxout}
\bibfield{author}{\bibinfo{person}{Ian~J Goodfellow}, \bibinfo{person}{David
  Warde-Farley}, \bibinfo{person}{Mehdi Mirza}, \bibinfo{person}{Aaron
  Courville}, {and} \bibinfo{person}{Yoshua Bengio}.}
  \bibinfo{year}{2013}\natexlab{}.
\newblock \showarticletitle{Maxout networks}.
\newblock \bibinfo{journal}{\emph{arXiv:1302.4389}} (\bibinfo{year}{2013}).
\newblock


\bibitem[\protect\citeauthoryear{Goodman and Flaxman}{Goodman and
  Flaxman}{2016}]%
        {goodman2016european}
\bibfield{author}{\bibinfo{person}{B. Goodman} {and} \bibinfo{person}{S.
  Flaxman}.} \bibinfo{year}{2016}\natexlab{}.
\newblock \showarticletitle{European Union regulations on algorithmic
  decision-making and a" right to explanation"}.
\newblock \bibinfo{journal}{\emph{arXiv:1606.08813}} (\bibinfo{year}{2016}).
\newblock


\bibitem[\protect\citeauthoryear{Harvey, Liaw, and Mehrabian}{Harvey
  et~al\mbox{.}}{2017}]%
        {harvey2017nearly}
\bibfield{author}{\bibinfo{person}{Nick Harvey}, \bibinfo{person}{Chris Liaw},
  {and} \bibinfo{person}{Abbas Mehrabian}.} \bibinfo{year}{2017}\natexlab{}.
\newblock \showarticletitle{Nearly-tight VC-dimension bounds for piecewise
  linear neural networks}.
\newblock \bibinfo{journal}{\emph{arXiv:1703.02930}} (\bibinfo{year}{2017}).
\newblock


\bibitem[\protect\citeauthoryear{He, Zhang, Ren, and Sun}{He
  et~al\mbox{.}}{2015}]%
        {he2015delving}
\bibfield{author}{\bibinfo{person}{K. He}, \bibinfo{person}{X. Zhang},
  \bibinfo{person}{S. Ren}, {and} \bibinfo{person}{J. Sun}.}
  \bibinfo{year}{2015}\natexlab{}.
\newblock \showarticletitle{Delving deep into rectifiers: Surpassing
  human-level performance on imagenet classification}. In
  \bibinfo{booktitle}{\emph{ICCV}}. \bibinfo{pages}{1026--1034}.
\newblock


\bibitem[\protect\citeauthoryear{Hinton, Vinyals, and Dean}{Hinton
  et~al\mbox{.}}{2015}]%
        {hinton2015distilling}
\bibfield{author}{\bibinfo{person}{Geoffrey Hinton}, \bibinfo{person}{Oriol
  Vinyals}, {and} \bibinfo{person}{Jeff Dean}.}
  \bibinfo{year}{2015}\natexlab{}.
\newblock \showarticletitle{Distilling the knowledge in a neural network}.
\newblock \bibinfo{journal}{\emph{arXiv:1503.02531}} (\bibinfo{year}{2015}).
\newblock


\bibitem[\protect\citeauthoryear{Hoyer}{Hoyer}{2002}]%
        {hoyer2002non}
\bibfield{author}{\bibinfo{person}{Patrik~O Hoyer}.}
  \bibinfo{year}{2002}\natexlab{}.
\newblock \showarticletitle{Non-negative sparse coding}. In
  \bibinfo{booktitle}{\emph{WNNSP}}. \bibinfo{pages}{557--565}.
\newblock


\bibitem[\protect\citeauthoryear{Kindermans, Hooker, Adebayo, Alber,
  Sch{\"u}tt, D{\"a}hne, Erhan, and Kim}{Kindermans et~al\mbox{.}}{2017}]%
        {kindermans2017reliability}
\bibfield{author}{\bibinfo{person}{Pieter-Jan Kindermans},
  \bibinfo{person}{Sara Hooker}, \bibinfo{person}{Julius Adebayo},
  \bibinfo{person}{Maximilian Alber}, \bibinfo{person}{Kristof~T Sch{\"u}tt},
  \bibinfo{person}{Sven D{\"a}hne}, \bibinfo{person}{Dumitru Erhan}, {and}
  \bibinfo{person}{Been Kim}.} \bibinfo{year}{2017}\natexlab{}.
\newblock \showarticletitle{The (Un) reliability of saliency methods}.
\newblock \bibinfo{journal}{\emph{arXiv:1711.00867}} (\bibinfo{year}{2017}).
\newblock


\bibitem[\protect\citeauthoryear{Koh and Liang}{Koh and Liang}{2017}]%
        {koh2017understanding}
\bibfield{author}{\bibinfo{person}{Pang~Wei Koh} {and} \bibinfo{person}{Percy
  Liang}.} \bibinfo{year}{2017}\natexlab{}.
\newblock \showarticletitle{Understanding black-box predictions via influence
  functions}.
\newblock \bibinfo{journal}{\emph{arXiv:1703.04730}} (\bibinfo{year}{2017}).
\newblock


\bibitem[\protect\citeauthoryear{Koiran and Sontag}{Koiran and Sontag}{1996}]%
        {koiran1996neural}
\bibfield{author}{\bibinfo{person}{Pascal Koiran} {and}
  \bibinfo{person}{Eduardo~D Sontag}.} \bibinfo{year}{1996}\natexlab{}.
\newblock \showarticletitle{Neural networks with quadratic VC dimension}. In
  \bibinfo{booktitle}{\emph{NIPS}}. \bibinfo{pages}{197--203}.
\newblock


\bibitem[\protect\citeauthoryear{Krizhevsky, Sutskever, and Hinton}{Krizhevsky
  et~al\mbox{.}}{2012}]%
        {krizhevsky2012imagenet}
\bibfield{author}{\bibinfo{person}{Alex Krizhevsky}, \bibinfo{person}{Ilya
  Sutskever}, {and} \bibinfo{person}{Geoffrey~E Hinton}.}
  \bibinfo{year}{2012}\natexlab{}.
\newblock \showarticletitle{Imagenet classification with deep convolutional
  neural networks}. In \bibinfo{booktitle}{\emph{NIPS}}.
  \bibinfo{pages}{1097--1105}.
\newblock


\bibitem[\protect\citeauthoryear{LeCun, Bengio, and Hinton}{LeCun
  et~al\mbox{.}}{2015}]%
        {lecun2015deep}
\bibfield{author}{\bibinfo{person}{Yann LeCun}, \bibinfo{person}{Yoshua
  Bengio}, {and} \bibinfo{person}{Geoffrey Hinton}.}
  \bibinfo{year}{2015}\natexlab{}.
\newblock \showarticletitle{Deep learning}.
\newblock \bibinfo{journal}{\emph{nature}} \bibinfo{volume}{521},
  \bibinfo{number}{7553} (\bibinfo{year}{2015}), \bibinfo{pages}{436}.
\newblock


\bibitem[\protect\citeauthoryear{Lee, Battle, Raina, and Ng}{Lee
  et~al\mbox{.}}{2007}]%
        {lee2007efficient}
\bibfield{author}{\bibinfo{person}{Honglak Lee}, \bibinfo{person}{Alexis
  Battle}, \bibinfo{person}{Rajat Raina}, {and} \bibinfo{person}{Andrew~Y Ng}.}
  \bibinfo{year}{2007}\natexlab{}.
\newblock \showarticletitle{Efficient sparse coding algorithms}. In
  \bibinfo{booktitle}{\emph{NIPS}}. \bibinfo{pages}{801--808}.
\newblock


\bibitem[\protect\citeauthoryear{Li, Chen, Hovy, and Jurafsky}{Li
  et~al\mbox{.}}{2015}]%
        {li2015visualizing}
\bibfield{author}{\bibinfo{person}{Jiwei Li}, \bibinfo{person}{Xinlei Chen},
  \bibinfo{person}{Eduard Hovy}, {and} \bibinfo{person}{Dan Jurafsky}.}
  \bibinfo{year}{2015}\natexlab{}.
\newblock \showarticletitle{Visualizing and understanding neural models in
  NLP}.
\newblock \bibinfo{journal}{\emph{arXiv:1506.01066}} (\bibinfo{year}{2015}).
\newblock


\bibitem[\protect\citeauthoryear{Mahendran and Vedaldi}{Mahendran and
  Vedaldi}{2015}]%
        {mahendran2015understanding}
\bibfield{author}{\bibinfo{person}{Aravindh Mahendran} {and}
  \bibinfo{person}{Andrea Vedaldi}.} \bibinfo{year}{2015}\natexlab{}.
\newblock \showarticletitle{Understanding deep image representations by
  inverting them}. In \bibinfo{booktitle}{\emph{CVPR}}.
  \bibinfo{pages}{5188--5196}.
\newblock


\bibitem[\protect\citeauthoryear{Montufar, Pascanu, Cho, and Bengio}{Montufar
  et~al\mbox{.}}{2014}]%
        {montufar2014number}
\bibfield{author}{\bibinfo{person}{Guido~F Montufar}, \bibinfo{person}{Razvan
  Pascanu}, \bibinfo{person}{Kyunghyun Cho}, {and} \bibinfo{person}{Yoshua
  Bengio}.} \bibinfo{year}{2014}\natexlab{}.
\newblock \showarticletitle{On the number of linear regions of deep neural
  networks}. In \bibinfo{booktitle}{\emph{NIPS}}. \bibinfo{pages}{2924--2932}.
\newblock


\bibitem[\protect\citeauthoryear{Nair and Hinton}{Nair and Hinton}{2010}]%
        {nair2010rectified}
\bibfield{author}{\bibinfo{person}{Vinod Nair} {and}
  \bibinfo{person}{Geoffrey~E Hinton}.} \bibinfo{year}{2010}\natexlab{}.
\newblock \showarticletitle{Rectified linear units improve restricted boltzmann
  machines}. In \bibinfo{booktitle}{\emph{ICML}}. \bibinfo{pages}{807--814}.
\newblock


\bibitem[\protect\citeauthoryear{Palm}{Palm}{2012}]%
        {IMM2012-06284}
\bibfield{author}{\bibinfo{person}{R.~B. Palm}.}
  \bibinfo{year}{2012}\natexlab{}.
\newblock \bibinfo{title}{Prediction as a candidate for learning deep
  hierarchical models of data}.
\newblock
\newblock


\bibitem[\protect\citeauthoryear{Pascanu, Montufar, and Bengio}{Pascanu
  et~al\mbox{.}}{2013}]%
        {pascanu2013number}
\bibfield{author}{\bibinfo{person}{Razvan Pascanu}, \bibinfo{person}{Guido
  Montufar}, {and} \bibinfo{person}{Yoshua Bengio}.}
  \bibinfo{year}{2013}\natexlab{}.
\newblock \showarticletitle{On the number of response regions of deep feed
  forward networks with piece-wise linear activations}.
\newblock \bibinfo{journal}{\emph{arXiv:1312.6098}} (\bibinfo{year}{2013}).
\newblock


\bibitem[\protect\citeauthoryear{Rather, Patel, and Khan}{Rather
  et~al\mbox{.}}{2017}]%
        {rather2017using}
\bibfield{author}{\bibinfo{person}{Nadeem~N Rather}, \bibinfo{person}{Chintan~O
  Patel}, {and} \bibinfo{person}{Sharib~A Khan}.}
  \bibinfo{year}{2017}\natexlab{}.
\newblock \showarticletitle{Using Deep Learning Towards Biomedical Knowledge
  Discovery}.
\newblock \bibinfo{journal}{\emph{IJMSC}} \bibinfo{volume}{3},
  \bibinfo{number}{2} (\bibinfo{year}{2017}), \bibinfo{pages}{1}.
\newblock


\bibitem[\protect\citeauthoryear{Ribeiro, Singh, and Guestrin}{Ribeiro
  et~al\mbox{.}}{2016}]%
        {ribeiro2016should}
\bibfield{author}{\bibinfo{person}{Marco~Tulio Ribeiro},
  \bibinfo{person}{Sameer Singh}, {and} \bibinfo{person}{Carlos Guestrin}.}
  \bibinfo{year}{2016}\natexlab{}.
\newblock \showarticletitle{Why should i trust you?: Explaining the predictions
  of any classifier}. In \bibinfo{booktitle}{\emph{KDD}}. ACM,
  \bibinfo{pages}{1135--1144}.
\newblock


\bibitem[\protect\citeauthoryear{Selvaraju, Das, Vedantam, Cogswell, Parikh,
  and Batra}{Selvaraju et~al\mbox{.}}{2016}]%
        {selvaraju2016grad}
\bibfield{author}{\bibinfo{person}{R.~R Selvaraju}, \bibinfo{person}{A. Das},
  \bibinfo{person}{R. Vedantam}, \bibinfo{person}{M. Cogswell},
  \bibinfo{person}{D. Parikh}, {and} \bibinfo{person}{D. Batra}.}
  \bibinfo{year}{2016}\natexlab{}.
\newblock \showarticletitle{Grad-cam: Why did you say that? visual explanations
  from deep networks via gradient-based localization}.
\newblock \bibinfo{journal}{\emph{arXiv:1610.02391}} (\bibinfo{year}{2016}).
\newblock


\bibitem[\protect\citeauthoryear{Shrikumar, Greenside, and Kundaje}{Shrikumar
  et~al\mbox{.}}{2017}]%
        {shrikumar2017learning}
\bibfield{author}{\bibinfo{person}{A. Shrikumar}, \bibinfo{person}{P.
  Greenside}, {and} \bibinfo{person}{A. Kundaje}.}
  \bibinfo{year}{2017}\natexlab{}.
\newblock \showarticletitle{Learning important features through propagating
  activation differences}.
\newblock \bibinfo{journal}{\emph{arXiv:1704.02685}} (\bibinfo{year}{2017}).
\newblock


\bibitem[\protect\citeauthoryear{Simonyan, Vedaldi, and Zisserman}{Simonyan
  et~al\mbox{.}}{2013}]%
        {simonyan2013deep}
\bibfield{author}{\bibinfo{person}{Karen Simonyan}, \bibinfo{person}{Andrea
  Vedaldi}, {and} \bibinfo{person}{Andrew Zisserman}.}
  \bibinfo{year}{2013}\natexlab{}.
\newblock \showarticletitle{Deep inside convolutional networks: Visualising
  image classification models and saliency maps}.
\newblock \bibinfo{journal}{\emph{arXiv:1312.6034}} (\bibinfo{year}{2013}).
\newblock


\bibitem[\protect\citeauthoryear{Smilkov, Thorat, Kim, Vi{\'e}gas, and
  Wattenberg}{Smilkov et~al\mbox{.}}{2017}]%
        {smilkov2017smoothgrad}
\bibfield{author}{\bibinfo{person}{D. Smilkov}, \bibinfo{person}{N. Thorat},
  \bibinfo{person}{B. Kim}, \bibinfo{person}{F. Vi{\'e}gas}, {and}
  \bibinfo{person}{M. Wattenberg}.} \bibinfo{year}{2017}\natexlab{}.
\newblock \showarticletitle{SmoothGrad: removing noise by adding noise}.
\newblock \bibinfo{journal}{\emph{arXiv:1706.03825}} (\bibinfo{year}{2017}).
\newblock


\bibitem[\protect\citeauthoryear{Sontag}{Sontag}{1998}]%
        {sontag1998vc}
\bibfield{author}{\bibinfo{person}{Eduardo~D Sontag}.}
  \bibinfo{year}{1998}\natexlab{}.
\newblock \showarticletitle{VC dimension of neural networks}.
\newblock \bibinfo{journal}{\emph{NATO ASI Series F Computer and Systems
  Sciences}}  \bibinfo{volume}{168} (\bibinfo{year}{1998}),
  \bibinfo{pages}{69--96}.
\newblock


\bibitem[\protect\citeauthoryear{Sundararajan, Taly, and Yan}{Sundararajan
  et~al\mbox{.}}{2017}]%
        {sundararajan2017axiomatic}
\bibfield{author}{\bibinfo{person}{Mukund Sundararajan}, \bibinfo{person}{Ankur
  Taly}, {and} \bibinfo{person}{Qiqi Yan}.} \bibinfo{year}{2017}\natexlab{}.
\newblock \showarticletitle{Axiomatic Attribution for Deep Networks}.
\newblock \bibinfo{journal}{\emph{arXiv:1703.01365}} (\bibinfo{year}{2017}).
\newblock


\bibitem[\protect\citeauthoryear{Wu, Hughes, Parbhoo, Zazzi, Roth, and
  Doshi-Velez}{Wu et~al\mbox{.}}{2018}]%
        {wu2017beyond}
\bibfield{author}{\bibinfo{person}{M. Wu}, \bibinfo{person}{M.~C Hughes},
  \bibinfo{person}{S. Parbhoo}, \bibinfo{person}{M. Zazzi}, \bibinfo{person}{V.
  Roth}, {and} \bibinfo{person}{F. Doshi-Velez}.}
  \bibinfo{year}{2018}\natexlab{}.
\newblock \showarticletitle{Beyond Sparsity: Tree Regularization of Deep Models
  for Interpretability}.
\newblock \bibinfo{journal}{\emph{AAAI}} (\bibinfo{year}{2018}).
\newblock


\bibitem[\protect\citeauthoryear{Xiao, Rasul, and Vollgraf}{Xiao
  et~al\mbox{.}}{2017}]%
        {xiao2017/online}
\bibfield{author}{\bibinfo{person}{Han Xiao}, \bibinfo{person}{Kashif Rasul},
  {and} \bibinfo{person}{Roland Vollgraf}.} \bibinfo{year}{2017}\natexlab{}.
\newblock \bibinfo{title}{Fashion-MNIST: a Novel Image Dataset for Benchmarking
  Machine Learning Algorithms}.
\newblock
\newblock
\showeprint[arXiv]{cs.LG/cs.LG/1708.07747}


\bibitem[\protect\citeauthoryear{Yosinski, Clune, Nguyen, Fuchs, and
  Lipson}{Yosinski et~al\mbox{.}}{2015}]%
        {yosinski2015understanding}
\bibfield{author}{\bibinfo{person}{J. Yosinski}, \bibinfo{person}{J. Clune},
  \bibinfo{person}{A. Nguyen}, \bibinfo{person}{T. Fuchs}, {and}
  \bibinfo{person}{H. Lipson}.} \bibinfo{year}{2015}\natexlab{}.
\newblock \showarticletitle{Understanding neural networks through deep
  visualization}.
\newblock \bibinfo{journal}{\emph{arXiv:1506.06579}} (\bibinfo{year}{2015}).
\newblock


\bibitem[\protect\citeauthoryear{Zemel, Wu, Swersky, Pitassi, and Dwork}{Zemel
  et~al\mbox{.}}{2013}]%
        {zemel2013learning}
\bibfield{author}{\bibinfo{person}{Rich Zemel}, \bibinfo{person}{Yu Wu},
  \bibinfo{person}{Kevin Swersky}, \bibinfo{person}{Toni Pitassi}, {and}
  \bibinfo{person}{Cynthia Dwork}.} \bibinfo{year}{2013}\natexlab{}.
\newblock \showarticletitle{Learning fair representations}. In
  \bibinfo{booktitle}{\emph{ICML}}. \bibinfo{pages}{325--333}.
\newblock


\bibitem[\protect\citeauthoryear{Zhou, Bau, Oliva, and Torralba}{Zhou
  et~al\mbox{.}}{2017}]%
        {zhou2017interpreting}
\bibfield{author}{\bibinfo{person}{Bolei Zhou}, \bibinfo{person}{David Bau},
  \bibinfo{person}{Aude Oliva}, {and} \bibinfo{person}{Antonio Torralba}.}
  \bibinfo{year}{2017}\natexlab{}.
\newblock \showarticletitle{Interpreting Deep Visual Representations via
  Network Dissection}.
\newblock \bibinfo{journal}{\emph{arXiv:1711.05611}} (\bibinfo{year}{2017}).
\newblock


\bibitem[\protect\citeauthoryear{Zhou, Khosla, Lapedriza, Oliva, and
  Torralba}{Zhou et~al\mbox{.}}{2016}]%
        {zhou2016learning}
\bibfield{author}{\bibinfo{person}{Bolei Zhou}, \bibinfo{person}{Aditya
  Khosla}, \bibinfo{person}{Agata Lapedriza}, \bibinfo{person}{Aude Oliva},
  {and} \bibinfo{person}{Antonio Torralba}.} \bibinfo{year}{2016}\natexlab{}.
\newblock \showarticletitle{Learning deep features for discriminative
  localization}. In \bibinfo{booktitle}{\emph{CVPR}}.
  \bibinfo{pages}{2921--2929}.
\newblock


\bibitem[\protect\citeauthoryear{Zhu, Shan, Mao, Yu, Rahmanian, and Zhang}{Zhu
  et~al\mbox{.}}{2017}]%
        {zhu2017deep}
\bibfield{author}{\bibinfo{person}{J. Zhu}, \bibinfo{person}{Y. Shan},
  \bibinfo{person}{JC Mao}, \bibinfo{person}{D. Yu}, \bibinfo{person}{H.
  Rahmanian}, {and} \bibinfo{person}{Y. Zhang}.}
  \bibinfo{year}{2017}\natexlab{}.
\newblock \showarticletitle{Deep embedding forest: Forest-based serving with
  deep embedding features}. In \bibinfo{booktitle}{\emph{KDD}}.
  \bibinfo{pages}{1703--1711}.
\newblock


\end{thebibliography}

\end{sloppy}
\end{document}